\documentclass[letterpaper]{article} 
\usepackage{aaai23}  
\usepackage{times}  
\usepackage{helvet}  
\usepackage{courier}  
\usepackage[hyphens]{url}  
\usepackage{graphicx} 
\usepackage{natbib}  
\usepackage{caption} 
\usepackage{algorithm}
\usepackage{algorithmic}

\usepackage[utf8]{inputenc} 

\usepackage{appendix}
\usepackage{amsthm}
\usepackage{amsmath}

\usepackage{amsfonts}
\usepackage{multirow}
\usepackage{multicol}

\usepackage{color}

\usepackage{xcolor,colortbl}
\definecolor{LightCyan}{rgb}{0.88,1,1}

\usepackage{subfigure}
\usepackage{tcolorbox}

\newtheorem{theorem}{Theorem}
\newtheorem{lemma}{Lemma}

\newtheorem{definition}{Definition}
\newtheorem{assumption}{Assumption}
\newtheorem{remark}{Remark}

\usepackage{newfloat}
\usepackage{listings}
\DeclareCaptionStyle{ruled}{labelfont=normalfont,labelsep=colon,strut=off} 
\floatstyle{ruled}
\newfloat{listing}{tb}{lst}{}
\floatname{listing}{Listing}
\title{Faster Adaptive Federated Learning}
\author {
    Xidong Wu\textsuperscript{\rm 1}, 
    Feihu Huang\textsuperscript{\rm 1,\rm 2}, 
    Zhengmian Hu\textsuperscript{\rm 1}, 
    Heng Huang\textsuperscript{\rm 1}
}
\affiliations {
    \textsuperscript{\rm 1} Electrical and Computer Engineering, University of Pittsburgh, Pittsburgh, PA, United States\\
    \textsuperscript{\rm 2} College of Computer Science and Technology, Nanjing University of Aeronautics  and Astronautics, Nanjing, China\\
    xidong\_wu@outlook.com, huangfeihu2018@gmail.com, 
    huzhengmian@gmail.com,
    henghuanghh@gmail.com
}

\usepackage{bibentry}

\begin{document}

\maketitle

\begin{abstract}
Federated learning has attracted increasing attention with the emergence of distributed data. While extensive federated learning algorithms have been proposed for the non-convex distributed problem, federated learning in practice still faces numerous challenges, such as the large training iterations to converge since the sizes of models and datasets keep increasing, and the lack of adaptivity by SGD-based model updates. Meanwhile, the study of adaptive methods in federated learning is scarce and existing works either lack a complete theoretical convergence guarantee or have slow sample complexity. In this paper, we propose an efficient adaptive algorithm (i.e., FAFED) based on the momentum-based variance-reduced technique in cross-silo FL. We first explore how to design the adaptive algorithm in the FL setting. By providing a counter-example, we prove that a simple combination of FL and adaptive methods could lead to divergence. More importantly, we provide a convergence analysis for our method and prove that our algorithm is the first adaptive FL algorithm to reach the best-known samples $O(\epsilon^{-3})$ and $O(\epsilon^{-2})$ communication rounds to find an $\epsilon$-stationary point without large batches. The experimental results on the language modeling task and image classification task with heterogeneous data demonstrate the efficiency of our algorithms.
\end{abstract}

\section{Introduction}
Distributed training, which emerges to address the challenge of distributed data, has attracted wide attention \cite{bao2022doubly}.
With the improvement of computing power, the bottleneck of training speed is gradually shifting from computing capacity to communication. Therefore, federated learning (FL) \cite{mcmahan2017communication} was proposed as an important distributed training paradigm in large-scale machine learning to reduce communication overhead.
In the FL setting, a central server coordinates multiple worker nodes to learn a joint model together with periodic model averaging by leveraging the massive local data of each worker node. The worker nodes share the computational load, and FL also provides some level of data privacy \cite{xiong2021privacy} because training data are not directly shared or aggregated.

More recently, an increasing number of FL works focus on addressing the cross-silo FL (i.e., FL between large institutions) problem \citep{xu2022coordinating, guo2022auto, karimireddy2020scaffold}, where most clients participate in computation every round and can maintain state between rounds. The cross-silo FL involves many practical applications, such as collaborative learning on financial data across various corporations and stakeholders or health data across numerous medical centers \citep{xu2022closing, guo2022auto}. 

In this paper, we consider solving a federated learning problem in the cross-silo setting, defined as 
\begin{align} \label{eq:1}
\min _{x \in \mathbb{R}^{d}}f(x):=\frac{1}{N} \sum_{i=1}^{N} f_{i}(x)
\end{align}
where $x \in \mathbb{R}^{d}$ denotes the model parameter and $N$ indicates the number of worker nodes. $f_{i}(x) = \mathbb{E}_{\xi^{(i)} \sim \mathcal{D}_{i}}\left[f_{i}\left(x ; \xi^{(i)}\right)\right]$ is the loss function of the $i^{th}$ worker node, and $\xi^{(i)} \sim \mathcal{D}_{i} $ denotes the samples $\xi^{(i)}$ drawn from distribution $\mathcal{D}_{i}$ on the $i^{th}$ worker node. When $\mathcal{D}_{i}$ and $\mathcal{D}_{j}$ are different ($i \neq j$ ), it is referred to as the heterogeneous data setting. In this paper, $\{\mathcal{D}_{i}\}_{i=1}^{N}$ are not identical. We restrict our focus to the non-convex problem, where the functions $f_i(x)$ and $f(x)$, therefore, are  smooth and non-convex. 
On worker node $i$, we have access to a stochastic gradient $\nabla f_{i}(x; \xi^{(i)})$ as an unbiased estimation of the $i^{th}$ worker node's true gradient $\nabla f_i(x)$. Worker nodes collaboratively learn a global model, but the raw data in each worker node is never shared with the server and other worker nodes.

Although, various FL methods have been
proposed \cite{karimireddy2020scaffold, reddi2020adaptive, hong2021federated, xiong2023, xiong2022towards}, which substantially reduce communication cost by avoiding frequent transmission between local worker nodes and the central server, it suffers from unfavorable convergence behavior. It is caused by a variety of factors, such as (1) client drift \cite{karimireddy2020scaffold}, where local client models move towards local optima instead of global optima, (2) lack of adaptivity as SGD-based update \cite{reddi2020adaptive}, and (3) large training iterations to converge as sizes of model parameters and training datasets keep increasing. Despite the recent advances, most of the existing work focuses on solving client drifts \cite{karimireddy2020scaffold,khanduri2021stem,xu2022coordinating}. 
The current federated learning framework
still cannot solve all challenges.

On the other hand, we know that adaptive methods have been widely developed and studied in non-federated settings since they generally require less parameter tuning and better convergence speed during the training phase. Meanwhile, like centralized method SGD, stochastic FL methods are not a good option for settings with heavy-tail stochastic gradient noise distributions. Such issues could be solved by adaptive learning rates \cite{zhang2019adam}, which combines knowledge of past iterations. In addition, adaptive gradient methods also escape saddle points faster compared to SGD \cite{staib2019escaping}. Therefore, the introduction of adaptive tools is an important direction to improve the performance of FL algorithms in the practice. 

However, the design of adaptive FL methods is nontrivial because the local worker node moves towards different directions and the global trackers cannot be updated frequently in the FL setting. The improper design of the adaptive FL method might lead to convergence issues \cite{chen2020toward}. \citet{reddi2020adaptive} firstly proposed a class of federated versions of adaptive optimizers, including FedAdagrad, FedYogi, and FedAdam. 
But its analysis only holds when the $\beta_1 = 0$ and it cannot use the advantage of the momentum. MimeAdam is proposed \citep{karimireddy2020mime} and it applies server statistics locally to address this issue. Nevertheless, MimeAdam has to compute the full local gradient, which might be forbidden in practice. More recently, FedAMS is proposed \citep{wang2022communication} and it provides the completed proof, but it doesn't improve the convergence rate. Overall, the sample convergence rates of FedAdagrad, FedYogi, FedAdam and FedAMS are $O\left(\epsilon^{-4}\right)$
(Not better than FedAvg). At the same time, they also require an extra global learning rate to tune.

As the sizes of model parameters and training datasets keep increasing, the deep learning models require more training iterations to converge and more efficient optimization methods are welcomed. Consequently, a natural question is whether one can achieve a faster convergence rate in theory and practice with adaptive technology. In this paper, we give an affirmative answer to the above question by proposing a faster adaptive FL algorithm (i.e., FAFED).

\textbf{Contributions} The main contributions of this work are listed below:

\begin{itemize}
\item 
We study how to incorporate the adaptive gradient method into federated learning. We propose a faster stochastic adaptive FL method (i.e., FAFED) in heterogeneous data settings based on the momentum-based variance reduction technique with a general adaptive matrix. 
\item We provide a convergence analysis
framework for our adaptive methods under some mild assumptions. Our algorithm is the first adaptive FL algorithm to reach the best-known samples complexity $O(\epsilon^{-3})$ and communication complexity $O(\epsilon^{-2})$ to find an $\epsilon$-stationary point without large batches. The extensive experimental results on the language modeling task and image classification
task confirm the effectiveness of our proposed algorithm. 
\item 
By establishing a counter example, we also show that a naive combination of adaptive gradient methods and periodic model averaging might result in divergence.
Therefore, sharing adaptive learning should be considered in the FL setting.
\end{itemize}

\begin{table*}
\centering

  \setlength{\tabcolsep}{12pt}
  \begin{tabular}{lcccc}
    \hline
    \textbf{Algorithm}  & \textbf{Reference}  & \textbf{Sample} &  \textbf{Communication} & \textbf{Adaptivity} \\
  \hline

FedAvg & \cite{yang2021achieving} & $O\left(\epsilon^{-4}\right)$ & $O\left( \epsilon^{-4}\right)$ & \\
& \cite{karimireddy2020scaffold} & &&\\
  \hline
FedAdagrad & \cite{reddi2020adaptive} & $O\left(\epsilon^{-4}\right)$ & $O\left(\epsilon^{-4}\right)$ & $\surd$\\ 
  \hline
FedYogi 
  & \cite{reddi2020adaptive} & $O\left(\epsilon^{-4}\right)$ & $O\left(\epsilon^{-4}\right)$ & $\surd$\\
  \hline
FedAdam & \cite{reddi2020adaptive}& $O\left(\epsilon^{-4}\right)$  & $O\left(\epsilon^{-4}\right)$ & $\surd$\\
      \hline
  FedAMS & \cite{wang2022communication}& $O\left(\epsilon^{-4}\right)$  & $O\left(\epsilon^{-4}\right)$ & $\surd$\\
      \hline
FAFED & Our work &$\Tilde{O} \left(\epsilon^{-3}\right)$ & $\Tilde{O}\left(\epsilon^{-2}\right)$ & $\surd$\\
  \hline
  \end{tabular}
\caption{Complexity comparison of FedAvg and typical adaptive FL algorithms for finding an $\epsilon$-stationary point.
  Sample complexity denotes the number of calls to the First-order Oracle (IFO) by all worker nodes to reach an $\varepsilon$-stationary point. Communication complexity is defined as the total number of back-and-forth communication rounds between each worker node and the central server required to reach an $\varepsilon$-stationary point. }\label{tb1}
\end{table*}

\section{Related Works}
\subsection{Federated Learning}
FedAvg was proposed in \cite{mcmahan2017communication} as the first FL algorithm. 
With periodic model averaging, it can dramatically reduce communication overheads.
Earlier works analyzed FL algorithms in the homogeneous data setting \cite{woodworth2020local, khaled2020tighter} and recent research extends federated learning to heterogeneous data settings (non-iid), as well as non-convex models, such as deep neural networks. When datasets on different worker nodes are homogeneous, FedAvg reduces to local SGD \cite{zinkevich2010parallelized}. 

Recent works \cite{yang2021achieving,karimireddy2020scaffold} consider FedAvg with partial worker nodes participation with $O(1)$ local updates iterations and batch sizes. The sample and communication complexities are both $O(\epsilon^{-4})$. In \cite{yu2019linear, yu2019parallel}, authors propose  Parallel Restarted SGD and Momentum SGD, and show that both of them require $O(\varepsilon^{-4}$) samples and $O(\varepsilon^{-3})$ rounds of communication to reach an $\varepsilon$-stationary solution. SCAFFOLD was proposed in \cite{karimireddy2020scaffold}, which uses control variates to correct for the ‘client-drift’ when the data is heterogeneous. It achieves the same sample and communication complexities as FedAvg. \citet{li2020federated} proposed a penalty-based method called FedProx to reduce the communication complexity to $O(\varepsilon^{-2})$. The analysis of FedProx depends on a gradient similarity assumption to restrict the data heterogeneity, 
which essentially requires that all minimums of $f(x)$ are also minimums of $f_i(x)$.
Later, FedPD was proposed in \cite{zhang2020fedpd} to relax this assumption.

Momentum-based optimizers are widely used in learning tasks \cite{sun2022demystify}. 
Subsequently, some momentum-based FL algorithms are proposed. For example, \citep{xu2022coordinating} introduces a momentum fusion technique to coordinate the server and local momentum buffers, but they do not reduce the complexity.  Based on variance reduction technology, Fed-GLOMO \cite{das2022faster} require $O(\varepsilon^{-3})$  sample complexity and $O(\varepsilon^{-3})$ communication complexity.
Their sample complexity matches the optimal complexity of the centralized non-convex stochastic optimization algorithms \cite{fang2018spider, cutkosky2019momentum}. More recently, STEM was proposed in \cite{khanduri2021stem} which utilizes a momentum-assisted stochastic gradient direction for both the worker nodes and central server updates. It further reduces the communication rounds to $O(\varepsilon^{-2}$) and keeps the same sample cost of $O(\varepsilon^{-3})$. 

\subsection{Adaptive Methods}
Adaptive methods are a class of optimization algorithms as one of the most important variants of stochastic gradient descent in machine learning. For example, Adam \cite{kingma2014adam} AdaGrad \cite{duchi2011adaptive}, AdaDelta  \cite{zeiler2012adadelta} are widely used as optimization tools in training deep neural networks (DNNs). Afterward, some variants \cite{reddi2019convergence} have been proposed to show a convergence guarantee in the non-convex setting. More recently, the works \cite{cutkosky2019momentum,huang2021super} presented some accelerated adaptive gradient methods based on the variance-reduced techniques. 

In FL settings, \citet{reddi2020adaptive} firstly propose federated versions of adaptive optimizers, including a class of adaptive FL methods, such as FedAdagrad, FedYogi, and FedAdam. These methods achieve the same sample cost and communication rounds as FedAvg when assuming the $\beta_1 = 0$.  \citet{chen2020toward} proposed Federated AMSGrad and achieves the same sample cost and communication rounds. MimeAdam is proposed in \citep{karimireddy2020mime} but it requires the full local gradient in each communication round. More recently, FedAMS is proposed in \citep{wang2022communication} and it provides the completed proof and considers the gradient compression. But it doesn't improve the convergence rate. Table \ref{tb1} summarizes the details of typical adaptive FL algorithms.

\section{Preliminaries}
\textbf{Notations}: 
For two vectors $x$ and $y$ in Euclidean space, $\langle x,y\rangle$ denote their inner product. $\|\cdot\|$ denotes
the $\ell_2$ norm for vectors and spectral norm for matrices, respectively. And $x_{t,i}$ denotes the local model parameters of the $i^{th}$ worker node at the iteration $t$.
$\nabla_x f(x)$ is the partial derivative w.r.t. variables $x$.  $I_d$ means $d$-dimension identity matrix. 
$a=O(b)$ denotes that $a \leq C b$ for some constant $C>0$, and the notation $\tilde{O}(\cdot)$ hides logarithmic terms.
Given the mini-batch samples $\mathcal{B}=\{\xi_i\}_{i=1}^q$, we let $\nabla f_i(x;\mathcal{B})=\frac{1}{q}\sum_{i=1}^q \nabla f_i(x;\xi_i)$.  

\begin{assumption} \label{ass:1}
(i) Unbiased Gradient. Each component function $f_i(x;\xi)$ computed at each worker node is unbiased  $\forall \xi^{(i)} \sim \mathcal{D}_i$, $i \in [N]$ and $x \in \mathbb{R}^{d}$:
\begin{align}
& \mathbb{E}[\nabla f_i(x;\xi)] = \nabla f_i(x), \nonumber 
\end{align}

(ii) Intra- and inter- node Variance Bound. The following holds for all $\xi^{(i)} \sim \mathcal{D}_i$, $i, j \in [N]$ and $x \in \mathbb{R}^{d}$: 
\begin{align}
& \mathbb{E}\|\nabla f_i(x;\xi^{(i)})-\nabla f_i(x)\|^2 \leq \sigma^2. \nonumber\\
& \left\|\nabla f_i(x)-\nabla f_j(x)\right\|^{2} \leq \zeta^{2} \nonumber
\end{align}
\end{assumption}

The assumption \ref{ass:1}-(ii) is a typical assumption used in FL algorithms to constrain the data heterogeneity. $\zeta$ is the heterogeneity parameter and represents the level of data heterogeneity. If datasets across each worker node have the identical distributions, i.e., $D_i = D_j$ for all $i, j \in [N]$, then we have $\zeta = 0$, corresponds to the  homogeneous data setting (I.I.D setting). In this paper, we consider the heterogeneous data setting and $\zeta > 0$.

\begin{assumption} \label{ass:2}
Each component function $f_i(x;\xi)$ has a $L$-Lipschitz gradient, i.e.,
$\forall x_1, x_2$, we have
\begin{align}
& \mathbb{E}\|\nabla_x f_i(x_1;\xi) - \nabla_x f(x_2;\xi)\| \leq L \|x_1 - x_2\|, \nonumber 
\end{align}
\end{assumption}
By using convexity of $\|\cdot\|$ and assumption \ref{ass:2}, we have 
\begin{align}
&\|\nabla_x f(x_1) - \nabla_x f(x_2)\| \nonumber \\&= \|\mathbb{E}\big[\nabla_x f(x_1;\xi) - \nabla_x f(x_2;\xi)\big]\| \nonumber\\
\leq& \mathbb{E} \|\nabla_x f(x_1;\xi) - \nabla_x f(x_2;\xi)\| \leq L \|x_1-x_2\| \nonumber
\end{align}
Assumption 2 is Lipschitz smooth, it is still a widely used assumption in optimization analysis. Many typical centralized stochastic algorithms use this assumption, such as SPIDER \cite{fang2018spider}, STORM \cite{cutkosky2019momentum}. Similarly, it is used in FL algorithms such as MIME \cite{karimireddy2020mime}, Fed-GLOMO \cite{das2022faster} and STEM\cite{khanduri2021stem}.

\begin{assumption} \label{ass:3}
The function $F(x)$ is bounded below in $\mathcal{X}$, \emph{i.e.,} $F^* = \inf_{x\in \mathcal{X}}F(x) > -\infty$.
\end{assumption}

\begin{assumption} \label{ass:4}
In our algorithms, the adaptive matrices $A_t$ for all $t\geq 1$ for updating the variables $x$ is a diagonal matrix and satisfies $ \lambda_{\min}(A_t)  \geq \rho >0$, where $\rho$ is an  appropriate positive number based on its definition.
\end{assumption}

Assumption \ref{ass:4} ensures that the adaptive matrices $A_t$, $\forall t\geq 1$, are positive definite, as in \citep{huang2021super}. The adaptive matrices $A_t$ are diagonal matrices, and we do not need to inverse the matrix $A_t$. 

\begin{assumption} \label{ass:5}
 (Bounded Gradients). The function $f_i(x)$ have $G$-bounded gradients, i.e., for any $i \in[N], x \in \mathbb{R}^{d}$, we have $\|\nabla f_i(x)\| \leq G$.
\end{assumption}

Assumption \ref{ass:5} is used to provide the upper bound of the gradient in the adaptive methods, as in  \citep{reddi2020adaptive, chen2020toward, wang2022communication}. It is a typical assumption in the adaptive methods to constrain the upper bound of the adaptive learning rate. It is reasonable and often satisfied in practice, for example, it holds for the finite sum problem.

\begin{definition}
 A point $x$ is called $\epsilon$-stationary point if $\|\nabla f(x)\| \leq \epsilon$.  Generally, a stochastic algorithm is defined to achieve an $\epsilon$-stationary point in $T$ iterations if  $\mathbb{E}\|\nabla f(x_T)\| \leq \epsilon$.
\end{definition}

\section{Faster Adaptive Federated Learning}

In this section, we explore how to design the method to combine adaptive gradient method with federated learning. We propose two algorithms to show the idea behind the design of the adaptive FL methods and how to use the adaptive learning rate properly. We use a counter example to show that naive combination of local adaptive update might result in divergence, and then propose our fast adaptive federated learning method (i.e., FAFED).

\subsection{Divergence of Local Adaptive Federated Learning}

The FedAdam, FedYogi, FedAdagrad proposed in \cite{reddi2020adaptive} and FedAMS proposed in \citep{wang2022communication} adjust the adaptive learning rate on the server. 
These methods have a main drawback that adaptive term cannot adjust the performance of the model in the local update, and introduce an extra global learning rate to tune.

To improve the algorithm, the most straightforward way to design an adaptive federated learning method is to add an adaptive term on each worker node and run an existing adaptive method, such as Adam, SuperAdam locally, and then average the model periodically after the inner loop. For ease of understanding, we design the adaptive method in algorithm \ref{alg:1}  based on FedAvg. Each work node runs local SGD with an adaptive learning rate independently. The model parameters $\{x_{t,i}\}_{i=1}^{N}$ are averaged after inner loop, as the FedAvg.

However, this design might suffer convergence issues and algorithm \ref{alg:1} can fail to converge to stationary points regardless of parameter selection \cite{chen2020toward}. It is because of heterogeneous data settings and the fact that the adaptive learning rates on different nodes are different. As a result, the global model moves away from the global optima point. Following \cite{chen2020toward}, theorem \ref{th:1} uses an example to present the details of step update in the algorithm \ref{alg:1} and shows that in some cases, divergence is unavoidable no matter how we choose the tuning parameters.

\begin{algorithm}[tb] 
\caption{Naive adaptive FedAvg Algorithm }
\label{alg:1}
\begin{algorithmic}[1] 
\STATE {\bfseries Input:} $T$, tuning parameters $\{\beta, \eta \}$, $v_{0, i}$ and mini-batch size $b_0$; \\
\STATE {\bfseries initialize:} Initialize: $\mathbf{x}_{i} \in \mathbb{R}^{d}$ for $i \in [N]$, 
\FOR{$t = 1, 2, \ldots, T$}
\STATE {\bfseries Client} i $\in [N]$: \\
\STATE  Draw mini-batch samples $\mathcal{B}_{t,i}=\{\xi_i^j\}_{j=1}^{b_0}$ with $|\mathcal{B}_t|=b_0$ from $D_i$ locally, and compute stochastic partial derivatives $\hat{g}_{t,i} = \nabla_x f_i (x_{t,i}; \mathcal{B}_{t,i})$
\\
\STATE $v_{t, i}=\beta v_{t-1, i}+\left(1-\beta\right) \hat{g}_{t, i}^{2}$ \\
\IF {$\mod(t,q)=0$}
\STATE Set  
$x_{t+1, i} = \bar{x}_{t+1, i} = \frac{1}{N} \sum_{j=1}^{N}\left(x_{t, j} - \eta \frac{\hat{g}_{t, j}}{\sqrt{v_{t, i}}}\right)$ \\
\ELSE
\STATE Set $x_{t+1, i} = x_{t, i} - \eta \frac{\hat{g}_{t, i}}{\sqrt{v_{t, i}}}$ \\ 
\ENDIF

\ENDFOR
\STATE {\bfseries Output:} $\bar{x}$ chosen uniformly random from $\{\bar{x}_t\}_{t=1}^{T}$.
\end{algorithmic}
\end{algorithm}

\begin{theorem} \label{th:1}
Suppose the sequence $\{\bar{x}_t\}_{t=1}^T$ are generated from algorithm \ref{alg:1} using stochastic partial derivatives. $\{\bar{x}_t\}_{t=1}^T$ might fail to converge to non-stationary points regardless of tuning parameter selection.
\end{theorem}
\begin{proof}
We utilize a counter example to prove theorem \ref{th:1} and consider a simple 1-dimensional case with N = 3 worker nodes as:

\vspace{-14pt}
\begin{align}
f_{1}=\left\{\begin{array}{ll}
3 x^{2}, & |x| \leq 1, \\
6|x|-2, & |x|>1
\end{array} \right. \nonumber
\end{align}
\vspace{-14pt}
\begin{align}
f_{2}=f_{3}= \begin{cases}- x^{2}, & |x| \leq 1 \\
-2|x|+1, & |x|>1\end{cases} \nonumber
\end{align}
\vspace{-14pt}
\begin{align}
f(x)= \frac{1}{3}\sum_{i=1}^{3} f_{i}(x)\left\{\begin{array}{ll} \frac{1}{3} x^{2}, & |x| \leq 1, \\
\frac{2}{3}|x|, & |x|>1
\end{array} \right. \nonumber
\end{align}

It is clear that $x = 0$ is the unique stationary point. We begin from step $t = 0$. Assume $\eta = 0.1, \beta = 0.5$ and $v_{0,i} = 0$ and the initial point is $x_{0,i}$ = 10 for $i = 1,2,3$.  With the first update (t = 1), for the $f_1$, we have $g_{0,1}$ = 6, and $v_{0,1}$ = $0.5 \times 6^2 =18$ . For $i=2,3$, we have $g_{0,i}$ = -2 and $v_{0,1}$ = 2. Following the algorithm \ref{alg:1} each worker node has its adaptive learning rate. Thus, after the first update, we have
$x_{1,1}= 10 - \frac{0.1}{3\sqrt{2}} \times 6 = 9.858$, and $x_{1,2} = x_{1,3} = 10 + \frac{0.1}{\sqrt{2}} \times 2$ = 10.14, and we have $\bar{x}= 10.05 $. 

The global model moves towards the opposite direction. We continue to show the following steps. We still have $g_{t,1}$ = 6 and $g_{t,2} = g_{t,3} = -2$. $v_{t,1}$ = $(1 - \beta^t) \times 6^2$ and $v_{t,2} = v_{t,3} = (1 - \beta^t) \times 2^2$. Therefore, $x_{t,1}$ always updates by $- 6\eta / \sqrt{(1 - \beta^t) \times 6^2}$ and $x_{t,2}, x_{t,3}$ always update by $ 2\eta / \sqrt{(1 - \beta^t) \times 2^2} $. As a result, the averaged model parameter will update $ \frac{\eta}{3\sqrt{(1 - \beta^t)}} $. It is the opposite of the direction of convergence and is independent of the choice of parameters $\eta$ and $\beta$. 

Therefore, after the first inner loop on each worker node and averaging step on the central server, the global model moves away from the optima point. In the following steps, each worker node continues running the local SGD with an adaptive learning rate from the same point. The global model keeps moving away from the optima point after each inner loop training. Finally, the global model fails to converge to the optimal point.

From this example, we could see the model diverges no matter what tuning parameters we choose. The divergence is caused by the heterogeneous data setting and the non-consensus of adaptive learning rates on different worker nodes. This suggests that we should combine the gradient information across nodes when we design the adaptive method in the FL setting. Thus, we use the sharing adaptive learning rate in the Algorithm \ref{alg:2} to avoid divergence. 
\end{proof}

\subsection{Faster Adaptive Federated Learning Method}
In the above subsection, we showed that SGD-based local adaptive learning method could diverge even in a very simple example regardless of tuning parameters selection. In this subsection, we propose a novel fast adaptive federated learning algorithm (FAFED) with shared adaptive learning rates for solving the problem under the heterogeneous data setting. Specifically, our FAFED algorithm is summarized in algorithm \ref{alg:2}.

At the step 8 in algorithm \ref{alg:2}, we use the coordinate-wise adaptive learning rate as in Adam \cite{kingma2014adam}, defined as:
\begin{align}
v_{t, i}=\beta v_{t-1, i}+\left(1-\beta\right) (\nabla_x f_i (x_{t,i}; \mathcal{B}_{t,i}))^{2}
\end{align}
where $\beta \in (0,1)$. At the step 10 in algorithm \ref{alg:2}, we add a periodic averaging step for local adaptive learning rate $v_{t,i}$ at the server side. Then we use $\bar{v}_{t,i}$ to generate an adaptive matrix $A_t =\mbox{diag}( \sqrt{\bar{v}_t} + \rho) $, where $\rho > 0$. In fact, the adaptive vector $v_t$ can be given with different adaptive learning rate methods, such as the global adaptive learning rate, AdaGrad-Norm \cite{ward2019adagrad}, and the $A_t$ keeps the same form. The tuning parameter $\rho$ is used to balance the adaptive information with noises. 

In the local update, different from algorithm \ref{alg:1}, algorithm \ref{alg:2} use the shared adaptive learning rates to avoid model divergence. At the step 15 in algorithm \ref{alg:2}, the same $A_t$ is used for local updates of different work nodes. The idea behind the design is that  $v_{t,i}$ can be viewed as the second-moment estimation of the gradients, thus $A_t$ established on the average of $v_{t,i}$ is also an estimation of the second moment of the global model. With the average of adaptive information, $A_t$ could follow the global direction and avoid the divergence issue in the algorithm \ref{alg:1}. 

At step 7 in algorithm \ref{alg:2}, we use the momentum-based variance reduced gradient estimator $m_{t,i}$, to track the gradient and update the model, defined as: 
\begin{align}
m_{t,i} =& \nabla_x f_i (x_{t,i}; \mathcal{B}_{t,i}) \nonumber\\
&+ (1-\alpha_{t})(m_{t-1} - \nabla_x f_i (x_{t-1,i}; \mathcal{B}_{t,i})
\end{align}
where $\alpha_t \in (0,1)$.
At the step 11 in algorithm \ref{alg:2}, 
the gradient estimator $m_{t,i}$ is also synchronized and averaged on the server. 

Overall, the local servers run adaptive updates locally with the shared adaptive learning rates, and the global server aggregates the model parameters, gradient estimator $m_{t,i}$ and the second-moment estimator $v_{t,i}$ every q steps. In the next section, we will establish the theoretical convergence guarantee of the proposed algorithm.

\subsection{Convergence Analysis of Our  Algorithm}
In this subsection, we study the convergence properties of our new algorithm under Assumptions \ref{ass:1}, \ref{ass:2}, \ref{ass:3}, \ref{ass:4}, and \ref{ass:5}. The details about proofs are provided in the supplementary materials.

Given the sequence $\{\bar{x}\}_{t=1}^{T}$ generated from our algorithms, we first define a useful convergence metric as follows:
\vspace{-10pt}
\begin{align}
\mathcal{M}_t 
=
\frac{1}{4\eta_{t}^{2}} \left\|\bar{x}_{t+1}-\bar{x}_{t}\right\|^{2}+\frac{1}{4 \rho^{2}}\left\|\nabla f\left(\bar{x}_{t}\right)-\bar{m}_{t}\right\|^{2}
\end{align}
where these two terms of $\mathcal{M}_t$ measure the convergence of the iteration solution of $\{\bar{x}\}_{t=1}^{T}$. The new convergence measure is tighter than the standard gradient norm metric, $\|\nabla f(\bar{x}_t)\|$, and we complete the final convergence analysis based on it.

\begin{algorithm}[tb]
\caption{FAFED Algorithm }
\label{alg:2}
\begin{algorithmic}[1] 
\STATE {\bfseries Input:} $T$, Parameters: $
\beta, \eta_t, \alpha_t$, the number of local updates $q$, 
and mini batch size $b$ and initial batch-size $B$; \\
\STATE {\bfseries initialize:} Initialize: $x_{0, i} = \bar{x}_0 = \frac{1}{N} \sum_{i=1}^{N} x_{0, i}$.  $m_{0, i} = \bar{m}_{0} = \frac{1}{N} \sum_{i=1}^{N} \hat{m}_{0, i}$ with $\hat{m}_{0, i} = \nabla_x f(x_{0, i};\mathcal{B}_{0, i})$ and $v_{0, i} = \bar{v}_{0} = \frac{1}{N} \sum_{i=1}^{N} \hat{v}_{0, i}$ with $\hat{v}_{0, i} = (\nabla_x f(x_{0, i};\mathcal{B}_{0, i}))^2$  where $|\mathcal{B}_{0,i}| = B$ from $D_i$ for $i \in [N]$. $A_0=\mbox{diag}( \sqrt{\bar{v}_0} + \rho)$ \\
\STATE $x_{1, i} = x_{0,i} - \eta_0  m_{0,i}$, for all $i \in [N]$ \\

\FOR{$t = 1, 2, \ldots, T$}
\STATE {\bfseries Client} i $\in [N]$: \\
\STATE  Draw mini-batch samples $\mathcal{B}_{t,i}=\{\xi_i^j\}_{j=1}^{b}$ with $|\mathcal{B}_{t,i}|=b$ from $D_i$ locally, and compute stochastic partial derivatives $\hat{g}_{t,i} = \nabla_x f_i (x_{t,i}; \mathcal{B}_{t,i})$ and $\hat{g}_{t-1, i} = \nabla_x f_i (x_{t-1,i}; \mathcal{B}_{t,i})$
\\
\STATE
$m_{t,i} = \hat{g}_{t,i} + (1-\alpha_{t})(m_{t-1} - \hat{g}_{t-1, i})  $
\\
\STATE $v_{t, i}=\beta v_{t-1, i}+\left(1-\beta \right) \hat{g}_{t, i}^{2}$ \\
\IF {$\mod(t,q)=0$}
\STATE ${v}_{t, i} = \bar{v}_{t}= \frac{1}{N} \sum_{i=1}^{N} v_{t, i}$ and $A_t=\mbox{diag}( \sqrt{\bar{v}_t} + \rho)$
\STATE $m_{t, i} = \bar{m}_{t} = \frac{1}{N} \sum_{i=1}^{N} m_{t, i}$
\STATE $x_{t+1, i} = \bar{x}_{t+1} = \frac{1}{N} \sum_{i=1}^{N}(x_{t, i} - \eta_t A_{t}^{-1} m_{t, i}) $\\
\ELSE
\STATE $A_t = A_{t-1}$
\STATE $x_{t+1, i} = x_{t, i} - \eta_t A_{t}^{-1} m_{t, i} $
\ENDIF
\ENDFOR
\STATE {\bfseries Output:} $\bar{x}$ chosen uniformly random from $\{\bar{x}_t\}_{t=1}^{T}$.
\end{algorithmic}
\end{algorithm}

\begin{theorem} \label{th:2}
Suppose that sequence $\{x_t\}_{t=1}^T$ are generated from algorithm \ref{alg:2}. Under the above Assumptions (\ref{ass:1},\ref{ass:2},\ref{ass:3},\ref{ass:4},\ref{ass:5}), given that $\forall t \geq 0$, $\alpha_{t+1}=c \eta^2_t$, $c =\frac{1}{12LI\bar{h}^3 \rho^2} + \frac{60L^2}{bN\rho^2} \leq \frac{120L^2}{bN\rho}$, 
$w$ = max$(\frac{3}{2}, w \leq 1728L^3I^3\bar{h}^3 - t)$
$\bar{h} = \frac{N^{2/3}}{L}$,  and set 
\begin{align} \label{eq:5}
\eta_t=\frac{\rho \bar{h}}{(w_t+t)^{1/3}}
\end{align}
then we have
\begin{align}
 &\frac{1}{T} \sum_{t=1}^{T} \mathbb{E}\|\nabla f(\bar{x}_{t})\| 
\leq G^{\prime} \sqrt{\frac{1}{T} \sum_{t=0}^{T-1}\mathbb{E}[\mathcal{M}_t]}  \nonumber\\
&\leq G^{\prime} \left[\left[\frac{12 L q}{\rho T}+\frac{L}{\rho(N T)^{2 / 3}}\right] \mathbb{E}\left[f\left(\bar{x}_{0}\right) - f^{*}\right] + \frac{6 q  \sigma^{2}}{T \rho^{2}}\right. \nonumber\\
&+ [\frac{12^{2} \times 150 q}{b^2\rho^{2} T}+\frac{1800}{b^2\rho^{2}(N T)^{2 / 3}}]\left[\frac{5 \sigma^{2}}{3}+\frac{3 \zeta^{2}}{2}\right] (\ln T+1) \nonumber\\
& + \left.\frac{\sigma^{2}}{2( N T)^{2 / 3 } \rho^{2}} \right]^{1/2} \nonumber
\end{align}
where $G^{\prime} = 4\sqrt{(\sigma^{2} + G^{2} + \rho^{2})}$ 
\end{theorem}

\begin{remark}
(Complexity) Without loss of generality, let $B = bq$ and $b = O(1) (b \geq 1)$, and choose $q=\left(T / N^{2}\right)^{1 / 3}$. Based on the definition of the $\varepsilon$-stationary point, namely, $\mathbb{E}\|\nabla f(x_T)\| \leq \epsilon$ and $\mathbb{E}\big[\mathcal{M}_T\big] \leq \epsilon^2$. we get $T = \tilde{O}(N^{-1}\varepsilon^{-3})$. And $\frac{T}{q} = (N T)^{2 / 3} = \tilde{O}(\varepsilon^{-2})$, 
Because the sample size b is a constant, the total sample cost is $\tilde{O}(N^{-1}\varepsilon^{-3})$ and the communication round is $\tilde{O}(\varepsilon^{-2})$ for finding an $\varepsilon$-stationary point that matches the state of the art of gradient complexity bound given in for solving the problem. And $\tilde{O}(N^{-1}\varepsilon^{-3})$ exhibits a linear speed-up compared with the aforementioned centralized optimal algorithms, such as SPIDER and STORM \cite{fang2018spider, cutkosky2019momentum}.
\end{remark}

\begin{remark}
(Data Heterogeneity) We use the $\zeta$ to present the data heterogeneity. From final results, it is shown that larger $\zeta$ (higher data heterogeneity) will slow down the training.  
\end{remark}

\begin{remark}
Due to Assumption \ref{ass:4} and the definition of $A_t$, the smallest eigenvalue of the adaptive matrix $A_t$ has a lower bound $\rho > 0$. 
It balances the adaptive information in the adaptive learning rate. Generally, we choose $\rho = O(1)$ and we do not choose a very small or large parameter in practice.
\end{remark}

\section{Experimental Results}

In this section, we evaluate our algorithms with language modeling task and image classification tasks. We compare our algorithms with the existing state-of-the-art algorithms, including FedAvg, SCAFFOLD \cite{karimireddy2020scaffold}, STEM \cite{khanduri2021stem}, FedAdam \cite{reddi2020adaptive} and FedAMS \cite{wang2022communication}.  Experiments are implemented using PyTorch, and we run all experiments on CPU machines with 2.3 GHz Intel Core i9 as well as NVIDIA Tesla P40 GPU.

\begin{figure*}[ht]
\centering
\subfigure[Training Loss]{
\hspace{0pt}
\includegraphics[width=.34\textwidth]{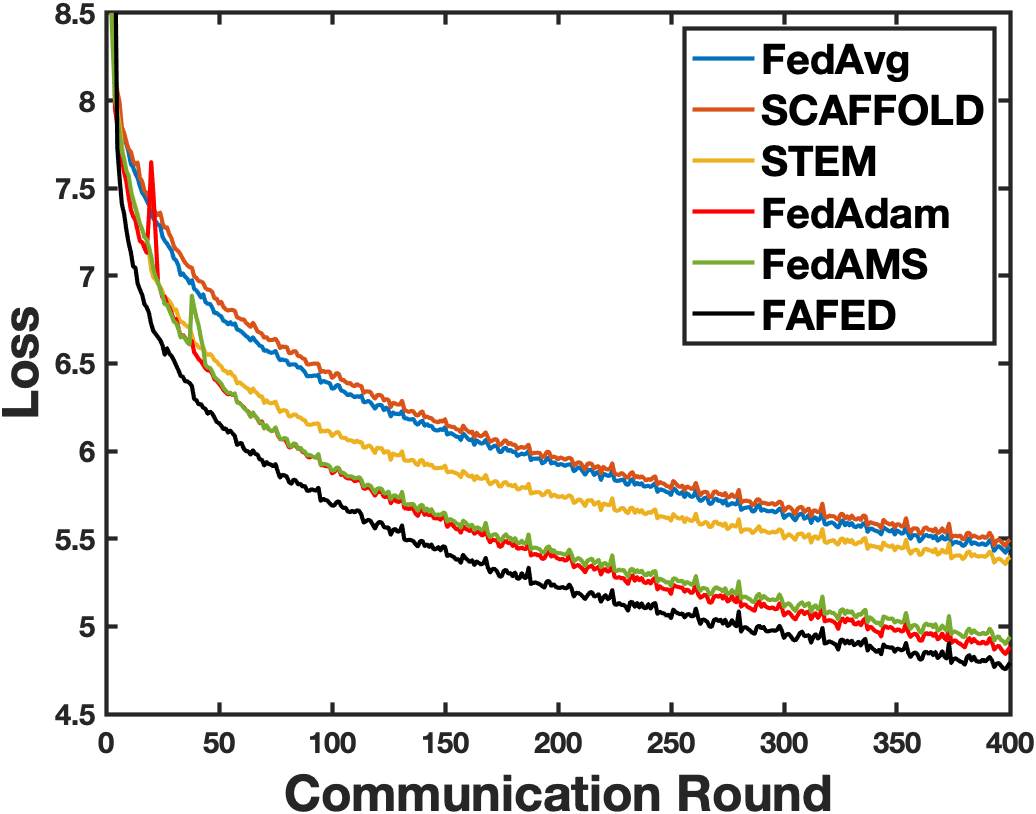}
}
\qquad
\subfigure[Training perplexities]{
\hspace{0pt}
\includegraphics[width=.34\textwidth]{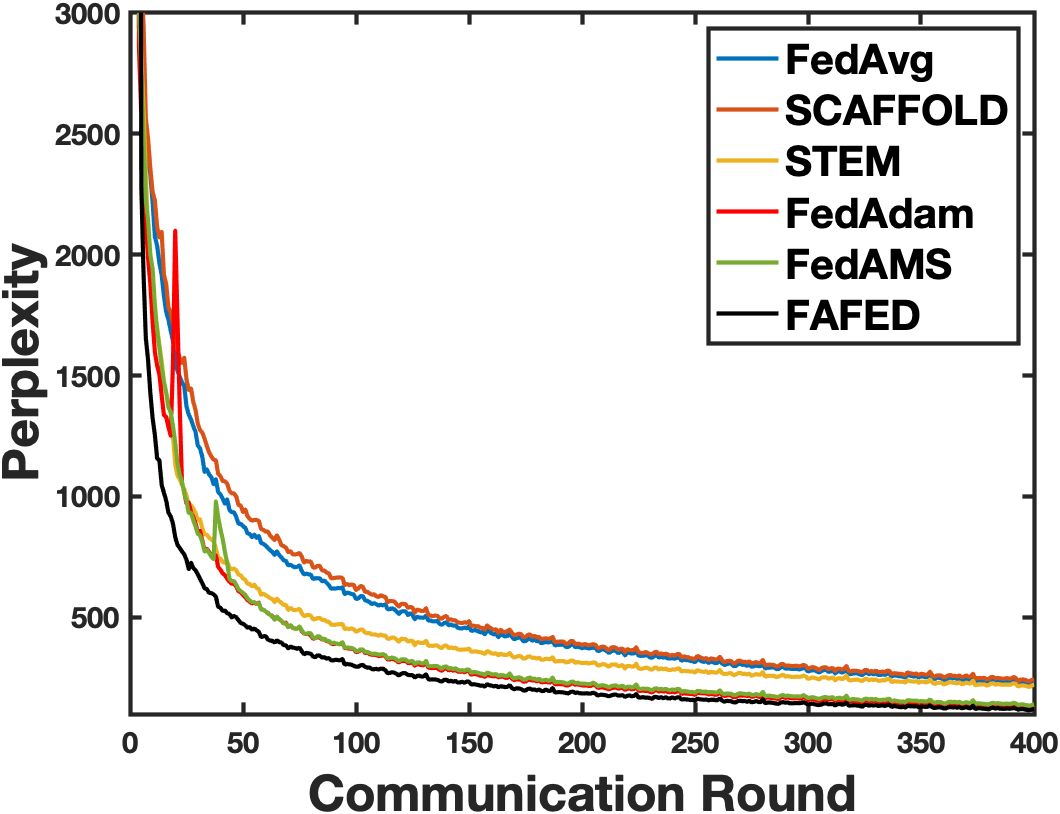}
}
\caption{Experimental results of WikiText2 for language modeling task.}
\label{fig:1}
\end{figure*}

\subsection{Language Modeling Task}
The WikiText2 dataset is used in the experiment and the data is partitioned by 16 worker nodes. Over the WikiText2 dataset, we train a 2-layer LSTM \cite{hochreiter1997long} with 650-dimensional word embeddings and 650 hidden units per layer. We used a batch size of 20 and an inner loop number $q$ of 10 in the experiment, and set the dropout rate as 0.5. To avoid the case of an exploding gradient in LSTM, we also clip the gradients by norm 0.25 \cite{huang2021super}. 

Grid search is used to choose learning rates for each optimizer.
In FedAvg, SCAFFOLD, FedAdam, and FedAMS, we set the learning rate as 10. The global learning rate in the SCAFFOLD is 1. Given that the large global learning rate in FedAdam and FedAms causes the divergence and big fluctuation, we tune it as $0.03 (\approx10^{-1.5})$. In STEM algorithm, we set $\bar{\kappa}$ as 20, $w = \sigma$ = 1, and the step-size is diminished in each epoch as in \cite{khanduri2021stem}. In FAFED algorithm, we set $\rho \bar{h}$ as 1 and $w$ = 1 and decrease the step size as \eqref{eq:5}. In the FedAdam, FedAMS, STEM and FAFED algorithm, the momentum parameters, such as $\alpha_t$, $\beta$, $\beta_1$ and $\beta_2$ are chosen from the set $\{0.1, 0.9 \}$. Their adaptive parameters $\tau$ or $\rho$ are chosen as 0.01. 

Figure \ref{fig:1} shows both training loss and training perplexities. Our FAFED algorithm outperforms all other baseline optimizers. Although FedAdam also has a good performance with an adaptive learning rate, it presents clear fluctuation at the beginning of the training phase because it utilizes the global adaptive method. FedAMS is worse because the adaptive term of FAFED and FedAdam is more flexible, while the adaptive term is monotonically increasing in FedAMS.

\begin{figure*}[t]
\centering
\subfigure[Fashion-MNIST]{
\hspace{0pt}
\includegraphics[width=.31\textwidth]{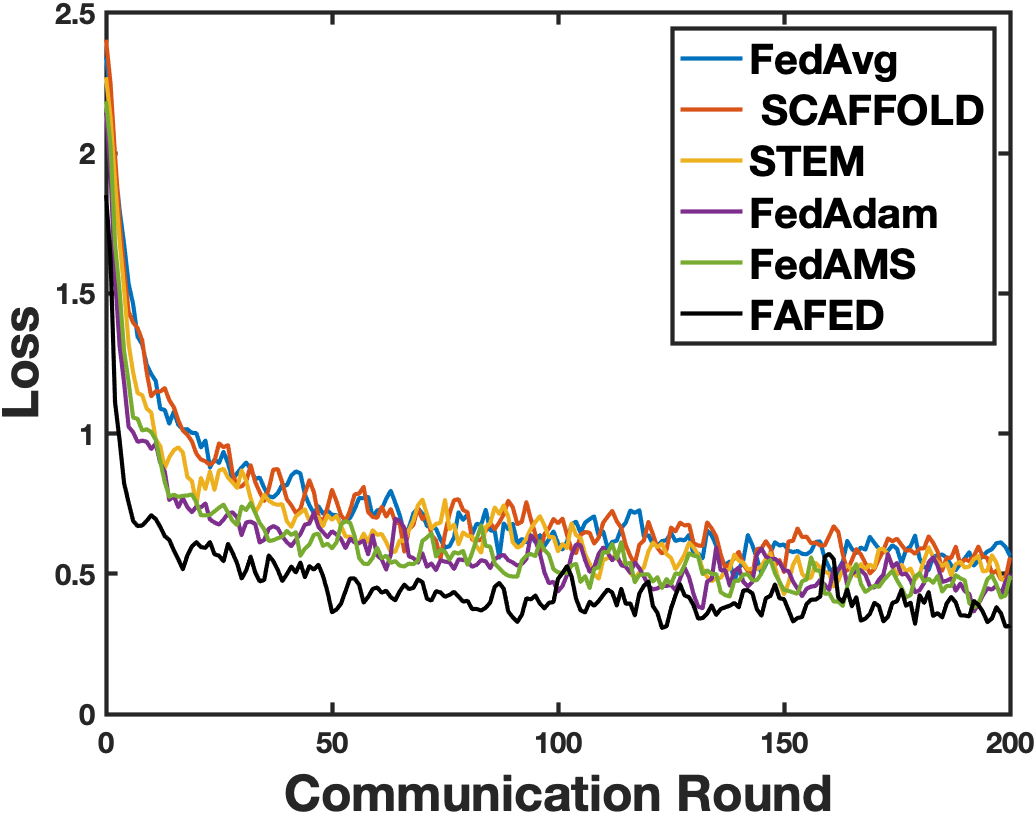}
}
\subfigure[MNIST]{
\hspace{0pt}
\includegraphics[width=.31\textwidth]{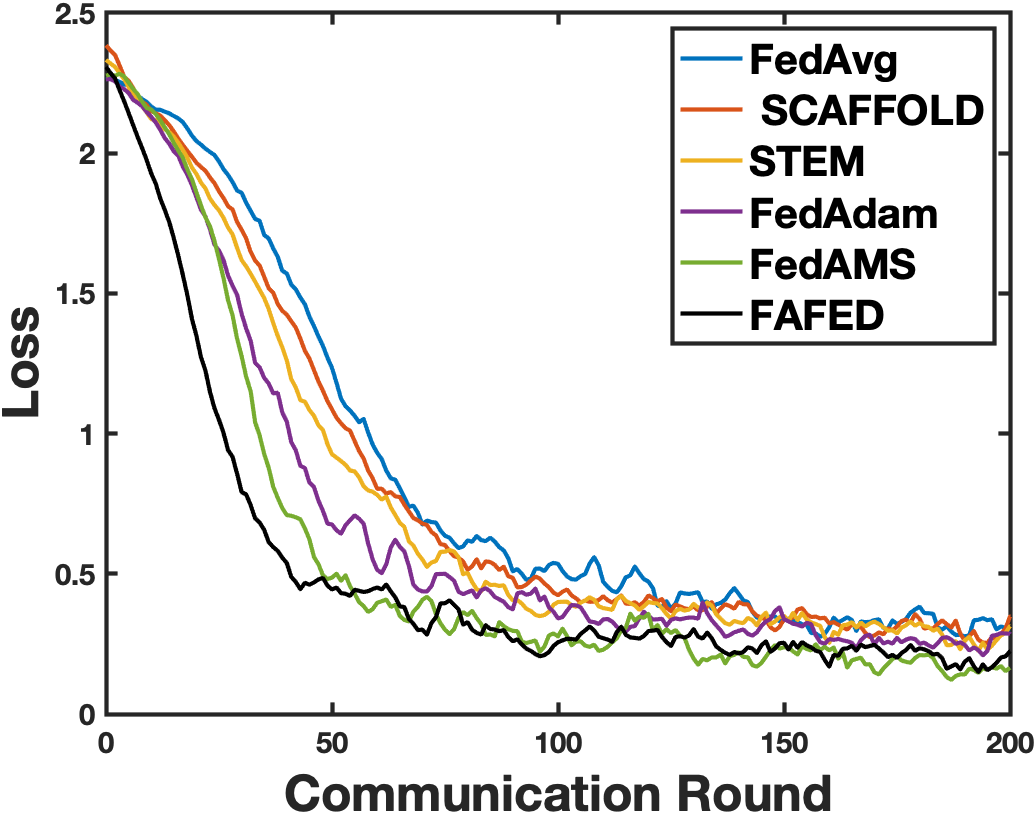}
}
\subfigure[CIFAR-10]{
\hspace{0pt}
\includegraphics[width=.31\textwidth]{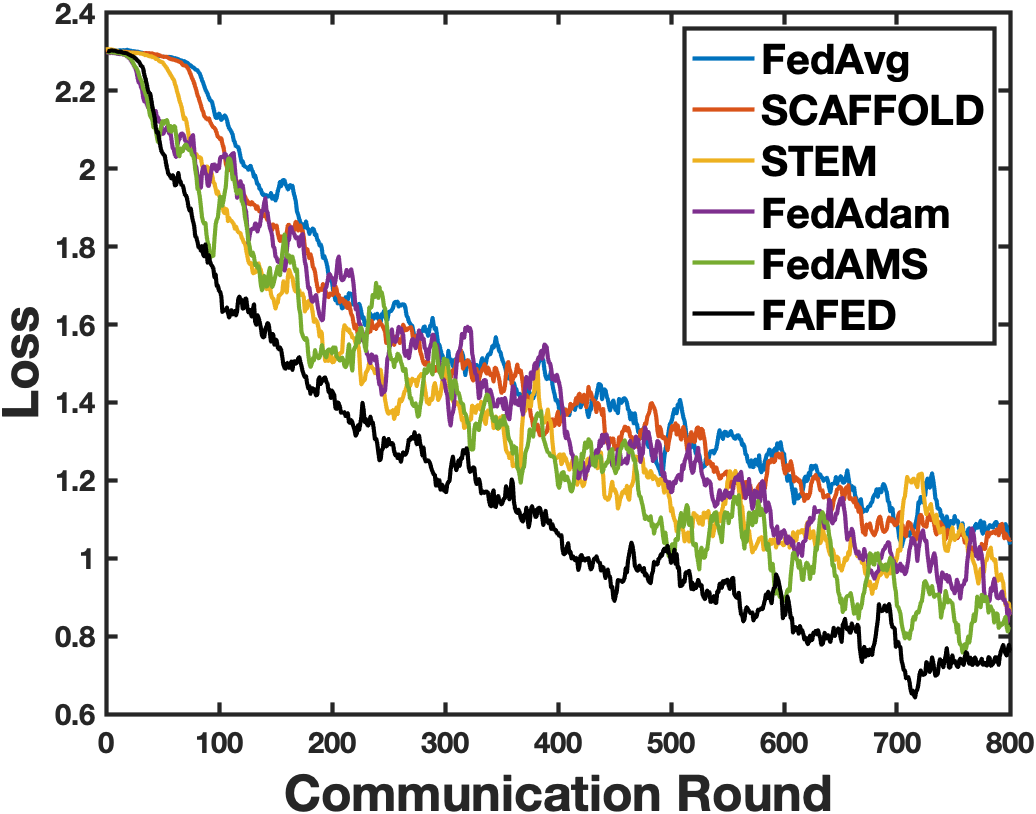}
}
\caption{Training loss vs the number of communication rounds for low heterogeneity setting.}
\label{fig:2}
\end{figure*}

 \subsection{Image Classification Tasks}

In the second task, we conduct image classification tasks on the Fashion-MNIST dataset as in \cite{nouiehed2019solving}, MNIST  dataset and CIFAR-10 dataset with 20 worker nodes in the network. The fashion-MNIST dataset and MNIST  dataset includes $60,000$ training images and $10,000$ testing images classified into $10$ classes. Each image in both datasets contains  $28\times 28$ arrays of the grayscale pixel. CIFAR-10 dataset includes $50,000$ training images and $10,000$ testing images. $60,000$ $32 \times 32$ color images are classified into $10$ categories. Each worker node holds the same Convolutional Neural Network (CNN) model as the classifier. We use cross entropy as the loss function. The network structures are provided in the supplementary material.

We consider three different heterogeneity settings: low heterogeneity, moderate heterogeneity and high heterogeneity.
For the low heterogeneity setting, datasets in different worker nodes have 95\% similarity \cite{karimireddy2020scaffold}. In the real-world application, the data on the different worker nodes are usually completely different, and they even have different categories. Then we consider two more challenging settings. For moderate heterogeneity settings and high heterogeneity settings, the datasets are divided into disjoint sets across all worker nodes. In the moderate heterogeneity setting, each worker node holds part of the data from all the classes, while for the high heterogeneity setting, each worker node only has a part of the total classes (5 out of 10 classes). 

We carefully tune hyperparameters for all methods. We run grid search for step size, and choose the step size in the set $\{0.001, 0.01, 0.02, 0.05,  0.1\}$. We set the global learning rate as 1 for SCAFFOLD. For FedAdam and FedAMS, we set global learning from the set $\{10^{-1.5}, 10^{-2}, 10^{-2.5},\}$, based on the Fig. 2 in \cite{reddi2020adaptive}. The adaptive parameter $\tau$ is chosen as 0.01. Given that the datasets are divided by all worker nodes, the number of data points in each worker node is limited. Heterogeneity settings also slow down the training. Thus, the number of local steps required increases, and the methods perform badly with diminishing step size because it decreases rapidly. We choose the fixed step size for STEM and FAFED. We choose the momentum parameter in the set $\{0.1, 0.9\}$. The $\beta_1$ and $\beta_2$ in FedAdam and FedAMS, and $\beta$ in FAFED are chosen from $ \{0.1, 0.9\}$. The batch-size $b$ is in $ \{5, 50, 100\}$ and the inner loop number $q \in \{5, 10, 20\}$. With the increase of data heterogeneity from low heterogeneity to high heterogeneity, we increase the batch size and decrease the inner loop number.

Discussion: The goal of our experiments is two-fold: (1) To compare the performance of FAFED with other algorithms in different heterogeneity settings during the training phase with training datasets; (2) To demonstrate the model performance on the test datasets. 

In Figures \ref{fig:2}, and figure (Seen in the supplementary materials), we show the performance of FAFED and other baseline methods against the number of communication rounds, namely back-and-forth communication rounds between each worker node and the central server on three datasets with different heterogeneity setting in the image classification task. From Figures, we can find that our algorithms consistently outperform the other baseline algorithms. Compared with FedAdam and FedAMS, our adaptive method has lower fluctuation (e.g., the beginning phase of FedAdam). That is because we use the adaptive learning rate locally, while FEDADAM just scales the model parameters in the central server after multistep training.

Finally, we focus on the performance on the testing datasets. In Tables (Seen in the supplementary mateirals), we show the testing accuracy of FAFED and that of other algorithms on the Fashion-MNIST dataset for different heterogeneity settings after training with the same epochs. Although, with the increasing data heterogeneity, model training becomes more difficult. FAFED performs well under all conditions. It shows the adaptive FL methods adapt well and our method (FAFED) has a good performance in different heterogeneity settings.

\section{Conclusion}
In this work, we proposed a novel adaptive algorithm (i.e., FAFED) based on the momentum-based variance reduced technique in the FL setting. We show that adaptive optimizers can be powerful tools and have a good performance in both theoretical analysis and numerical experiments. In the beginning, we explore how to design the adaptive algorithm in the FL setting. By providing a counter example, we present that a naive combination of the local adaptive method with the periodic model average can lead to divergence, and sharing adaptive learning should be considered. Moreover, we provide a solid convergence analysis for our methods, and prove that our algorithm is the first adaptive FL method to reach the best-known samples complexity $O(\epsilon^{-3})$ and communication complexity $O(\epsilon^{-2})$ to find an $\epsilon$-stationary point without large batches. Finally, we conduct experiments on the language modeling task and image classification tasks with different levels of heterogeneous data. 

\section*{Acknowledgements}
This work was partially supported by NSF IIS 1838627, 1837956, 1956002, 2211492, CNS 2213701, CCF 2217003, DBI 2225775. 


\bibliography{aaai23}

\newpage
\onecolumn
\section{Partial Results for Image Classification Tasks}
\vspace*{-10pt}
\begin{figure*}[h]
\centering
\subfigure[Fashion-MNIST]{
\hspace{0pt}
\includegraphics[width=.31\textwidth]{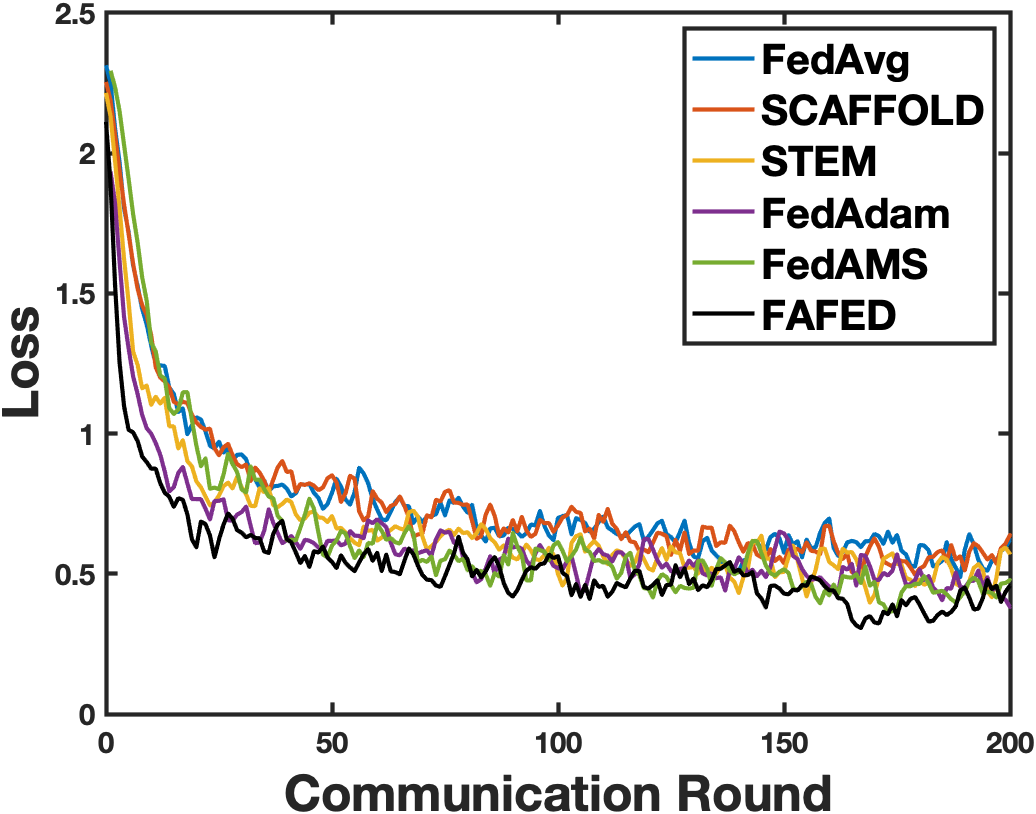}
}
\subfigure[MNIST]{
\hspace{0pt}
\includegraphics[width=.31\textwidth]{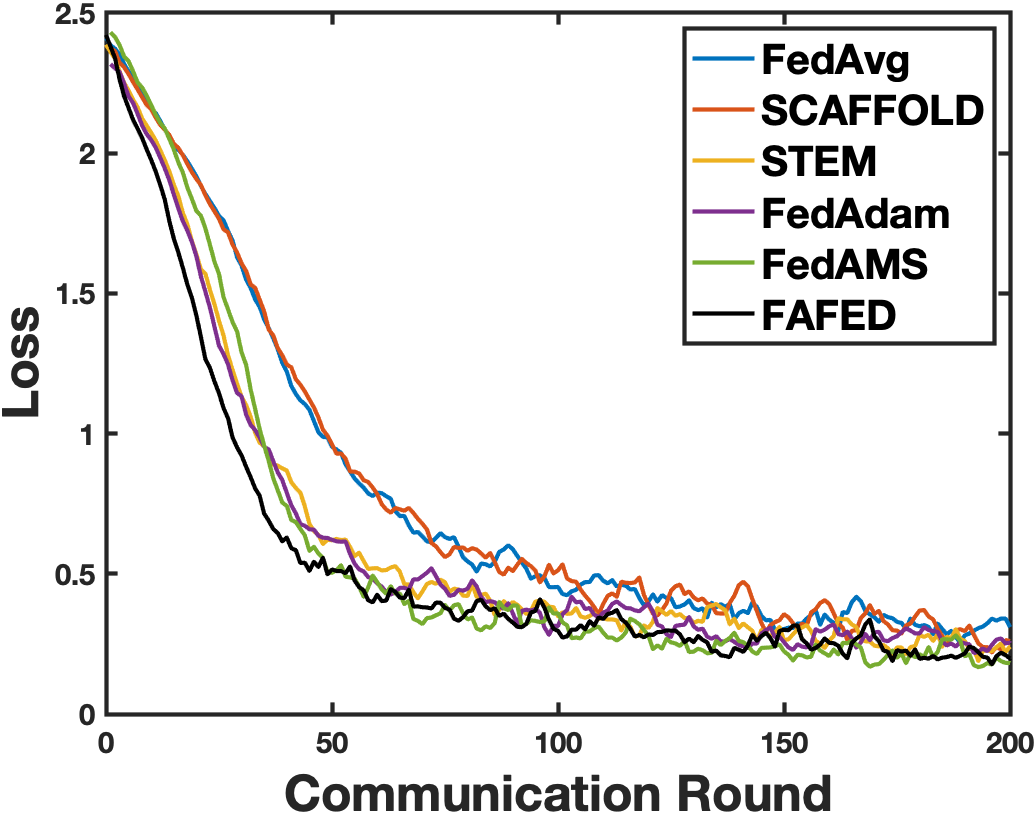}
}
\subfigure[CIFAR-10]{
\hspace{0pt}
\includegraphics[width=.31\textwidth]{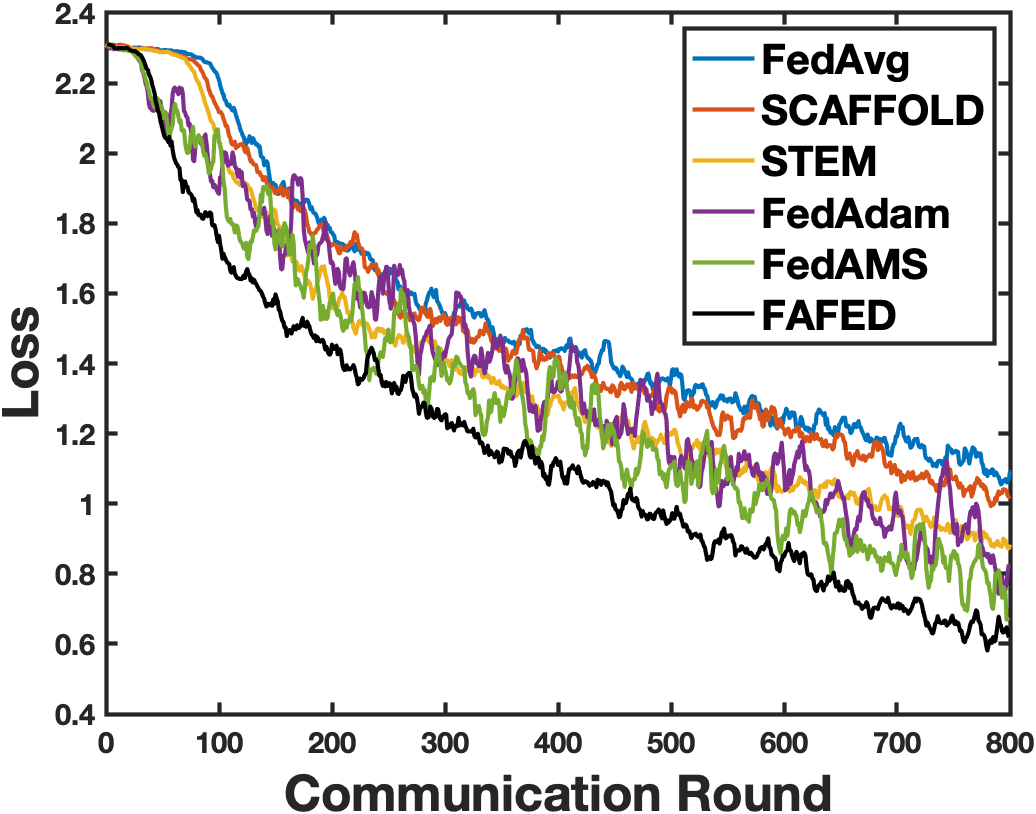}
}
\vspace*{-6pt}
\caption{Training loss vs the number of communication rounds for medium heterogeneity setting.}
\label{fig:3}
\vspace*{-6pt}
\end{figure*}
\vspace{-10pt}
\begin{figure*}[h]
\centering
\subfigure[Fashion-MNIST]{
\hspace{0pt}
\includegraphics[width=.31\textwidth]{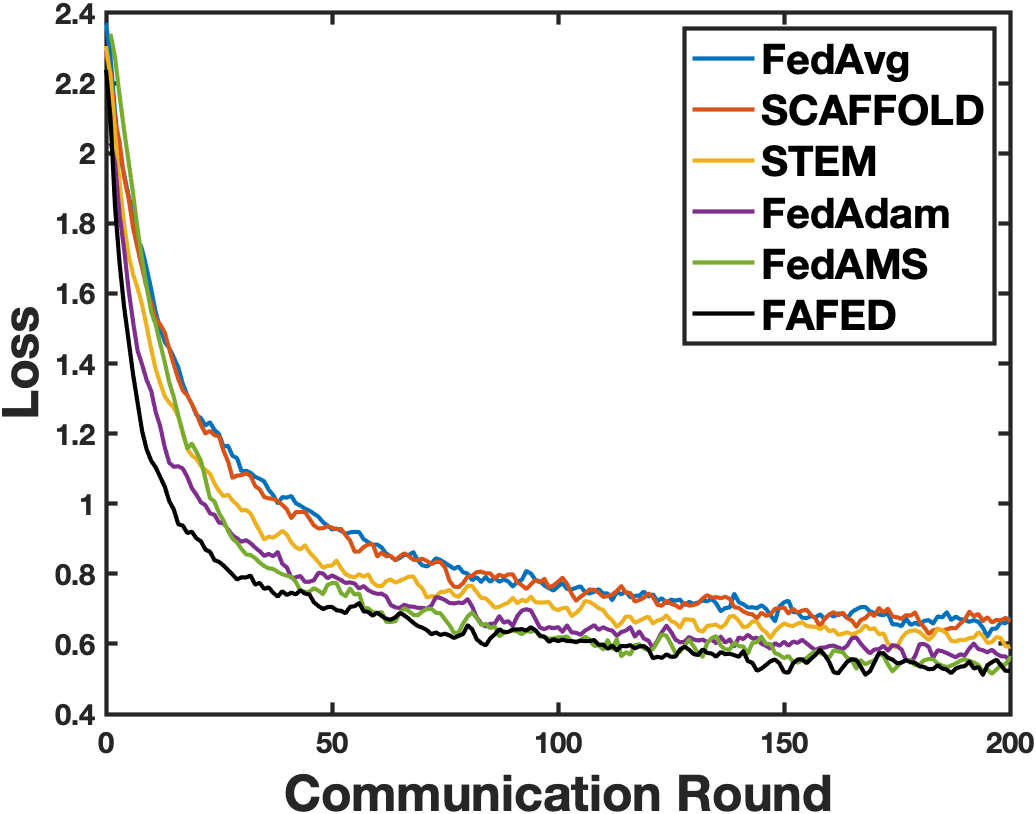}
}
\subfigure[MNIST]{
\hspace{0pt}
\includegraphics[width=.31\textwidth]{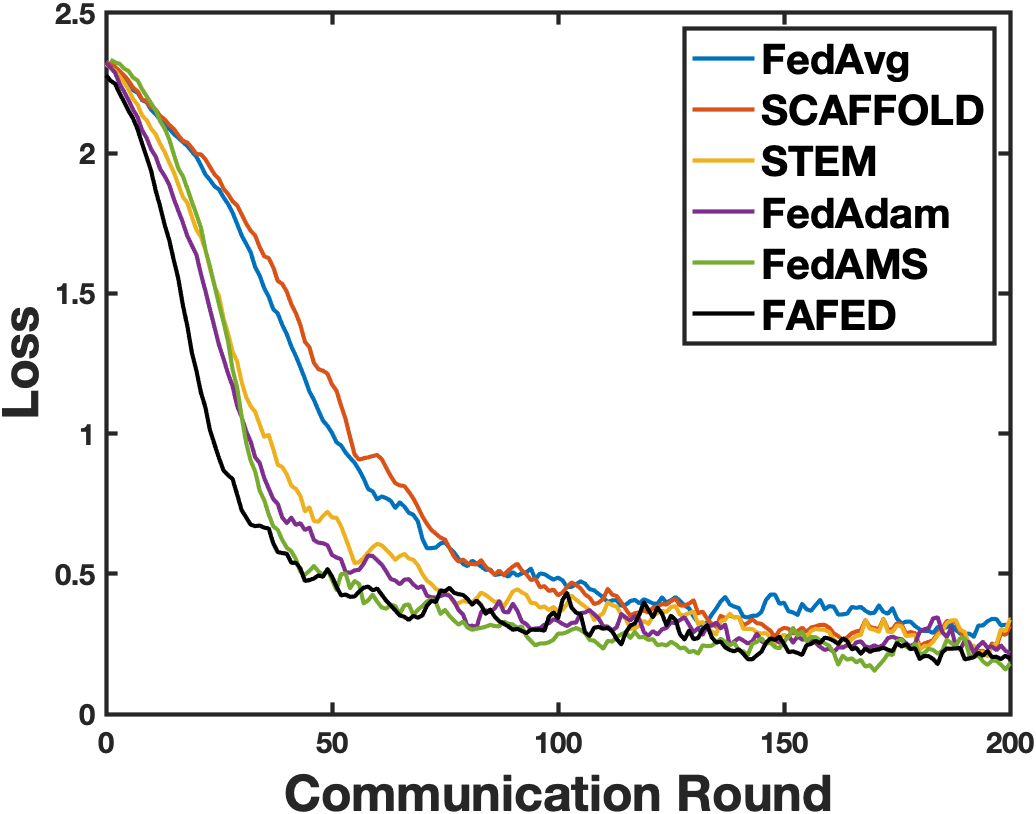}
}
\subfigure[CIFAR-10]{
\hspace{0pt}
\includegraphics[width=.31\textwidth]{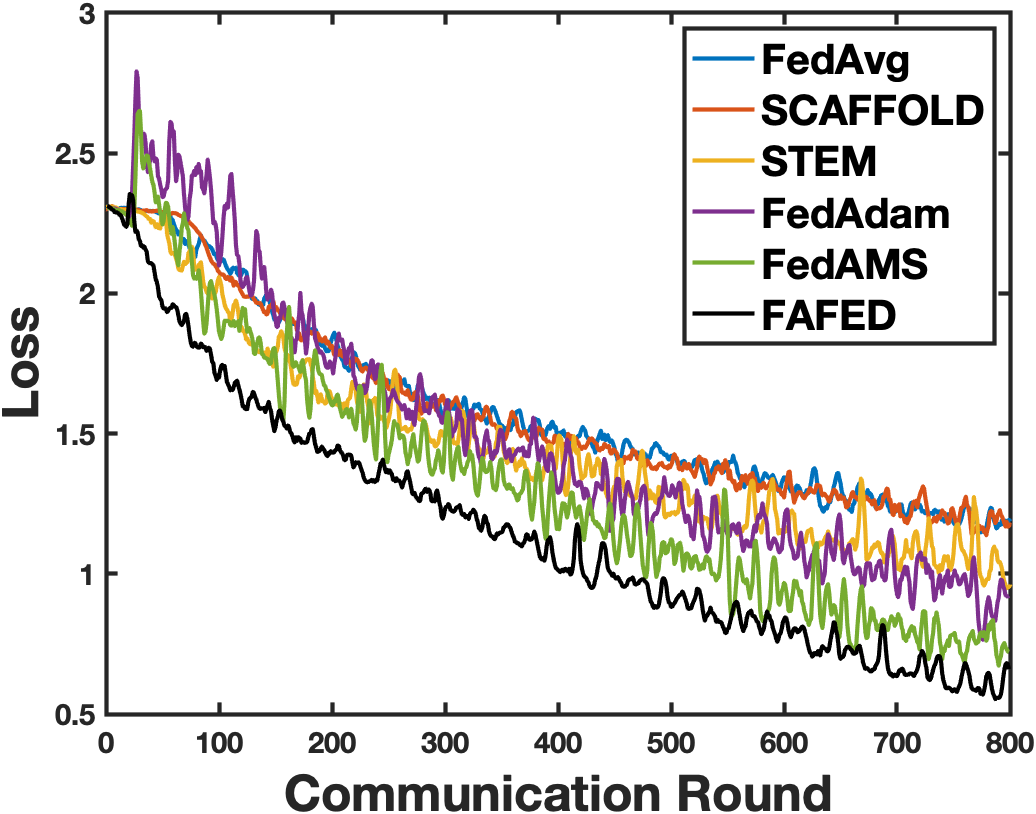}
}
\vspace*{-6pt}
\caption{Training loss vs the number of communication rounds for high heterogeneity setting.}
\label{fig:4}
\vspace*{-6pt}
\end{figure*}
\begin{table}[H]
    \centering
    \setlength{\tabcolsep}{20pt}
    \begin{subtable}[Low heterogeneity]{
    \begin{tabular}{l | c }
        \hline
        Algorithms & Accuracy \\
        \hline 
        FedAvg & 0.8451 \\
        SCAFFOLD & 0.8496\\
        STEM & 0.8562\\
        FedAdam & 0.8586\\
        FedAMS & 0.8697\\
        FAFED & \textbf{0.8816}\\
        \hline
      \end{tabular}}
      \label{tab:2:1}
    \end{subtable}
    \vspace*{-6pt}
    \\
    \centering
    \begin{subtable}[Medium heterogeneity]{        \begin{tabular}{l | c }\hline
        Algorithms & Accuracy \\
        \hline 
        FedAvg & 0.8454\\
        SCAFFOLD & 0.8461\\
        STEM & 0.8551\\ 
        FedAdam & 0.8581\\
        FedAMS & 0.8615\\
        FAFED & \textbf{0.8654}\\
        \hline
      \end{tabular}}
    \end{subtable}
    \vspace*{-6pt}
    \\
    \centering
    \begin{subtable}[High heterogeneity]{        \begin{tabular}{l | c }\hline
        Algorithms &  Accuracy \\
        \hline 
        FedAvg & 0.7958\\
        SCAFFOLD & 0.8034\\
        STEM & 0.8053 \\
        FedAdam & 0.8040\\ 
        FedAMS & 0.8015\\ 
        FAFED & \textbf{0.8188}\\
        \hline
      \end{tabular}} 
    \end{subtable}
    \vspace*{-10pt}
     \caption{Testing accuracy of different algorithms on Fashion-MNIST dataset for different heterogeneity settings.}
     \label{tab:2}
\end{table}

Deep learning has various applications, such as natural language processing \cite{devlin2018bert, guo2020inflating, dou2023measurement}, image classification \cite{he2019bag, wu2022retrievalguard, wu2022adversarial, dou2022optimal, wu2023decentralized} and speech recognition \cite{ nassif2019speech, zhou2022multi}, which attract much attention to propose different optimization methods \cite{fang2018spider, cutkosky2019momentum, bao2020fast, wu2022fast} to improve the performance of models training.

Here we provide partial results in image classification tasks. Figure \ref{fig:3} shows the results in the medium heterogeneity setting and Figure \ref{fig:4} shows the results in the high heterogeneity setting. Table \ref{tab:2} shows the performance of methods on the Fashion-MNIST testing datasets.

\section{Basic Lemma}
In this section, we provide the detailed convergence analysis of our algorithm. For convenience, in the subsequent analysis, we define $g_{t,i} = \nabla f_i(x_{t,i})$ and $a_t = [a_{t,1}^{\top}, a_{t,2}^{\top}, \cdots, a_{t,N}^{\top}]^{\top} \in \mathbb{R}^{Nd}$ and $\bar{a}_t = \sum_{i=1}^{N} a_i \in \mathbb{R}^{d}$ for $a_t \in \{ x_t, m_t, g_t \}$. $\otimes$ denotes the Kronecker product and $s_t$ denotes the $s_{t}=\lfloor t / q\rfloor q$. 


\begin{lemma} \label{lem:A1} 
For $x_t \in \mathbb{R}^{Nd}$ and $\bar{x}_t \in \mathbb{R}^{d}$, denoting $\mathbf{1} \in \mathbb{R}^{N}$ as the vector of all ones, we have 
\begin{align}
\|x -  \mathbf{1} \otimes \bar{x} \|^{2} \leq \|x\|^{2}
\end{align}
\end{lemma}
\begin{proof}
\begin{align}
\|x - \mathbf{1} \otimes \bar{x} \|^{2} &= \|(\mathbf{I} - \frac{1}{N}\mathbf{1}\mathbf{1}^{\top})x \|^2  \nonumber\\
&\leq \|\mathbf{I} - \frac{1}{N}\mathbf{1}\mathbf{1}^{\top}\|^2  \|x \|^2 \nonumber\\
&\leq \|x\|^{2}
\end{align}
where the first inequality follows from the Cauchy-Schwarz inequality and the last inequality follows the fact that $\|\mathbf{I} - \frac{1}{N}\mathbf{1}\mathbf{1}^{\top}\| \leq 1$
\end{proof}

According to algorithm \ref{alg:2}, we could update the model parameter as follows :
\begin{align}
\bar{x}_{t+1}
=\bar{x}_{t}-\eta_{t} A_{t}^{-1} \bar{m}_{t}
\end{align}

Then we consider the upper bound of $A_t$
\begin{lemma} \label{lem:A2}
Suppose the adaptive matrices sequence $\{A_t\}_{t=1}^{T}$ be generated from Algorithm \ref{alg:2}. Under the above Assumptions, we have 
 $\mathbb{E}\left\|A_{t}\right\|^{2} \leq 2\left(\sigma^{2}+G^{2}+\rho^{2}\right)$.
\end{lemma}
\begin{proof}
\begin{align}
&\mathbb{E}\|A_{t}\|^{2} = \mathbb{E}\|A_{s_{t}}\|^{2} = \mathbb{E}\|\sqrt{v_{s_t}}+\rho\|_{\infty}^{2} \leq 2 \mathbb{E}\|\bar{v}_{s_t}\|_{\infty} + 2 \rho^{2}
\end{align}
where the second inequality holds by that spectral norm of a matrix $A_{t}$ is the largest singular value of $A_{t}$ (.i.e, $\|A_{t}\|_2=\sqrt{\lambda_{\max }\left(A_{t}^* A_{t}\right)}=\sigma_{\max }(A_{t})$) and $A_{t}$ is a diagnoal matrix. 
Following the definition of $v_{s_t}$, we have 
\begin{align}
\mathbb{E}\|\bar{v}_{s_t}\|_{\infty} &= \mathbb{E}\|\beta \bar{v}_{s_{t}-1} + (1 - \beta) \frac{1}{N} \sum_{i=1}^{N}[ \nabla_x f_i (x_{s_{t}-1,i}; \mathcal{B}_{s_{t}-1,i})]^2\|_{\infty} \nonumber \nonumber\\
&\leq \beta \mathbb{E}\left\|\bar{v}_{s_{t}-1}\right\|_{\infty} + (1-\beta) \mathbb{E}\|\frac{1}{N} \sum_{i=1}^{N} [\nabla_x f_i (x_{s_{t}-1,i}; \mathcal{B}_{s_{t}-1,i})]^2\|_{\infty} \nonumber\\
&\leq \beta \mathbb{E}\left\|\bar{v}_{s_{t}-1}\right\|_{\infty} + (1-\beta) \frac{1}{N} \sum_{i=1}^{N}\mathbb{E}\| [\nabla_x f_i (x_{s_{t}-1,i}; \mathcal{B}_{s_{t}-1,i})]^2\|_{\infty} \nonumber\\
&\leq \beta \mathbb{E}\|\bar{v}_{s_{t}-1}\|_{\infty} + (1-\beta) \frac{1}{N} \sum_{i=1}^{N} \mathbb{E}\| \nabla_x f_i (x_{s_{t}-1,i}; \mathcal{B}_{s_{t}-1,i})\|_{2}^{2}
\end{align}
Then for the last term, we have 
\begin{align}
&\frac{1}{N} \sum_{i=1}^{N} \mathbb{E}\|\nabla_x f_i (x_{s_{t}-1,i}; \mathcal{B}_{s_{t}-1,i})\|^{2} \nonumber\\
\leq& \frac{1}{N} \sum_{i=1}^{N} [2 \mathbb{E}\|\nabla_x f_i (x_{s_{t}-1,i}; \mathcal{B}_{s_{t}-1,i}) - g_{s_{t}-1,i}\|^{2} + 2 \mathbb{E}\|g_{s_{t}-1,i}\|^2] \nonumber\\
\leq& 2 \sigma^{2} + 2 G^{2}
\end{align}
where the last inequality follows the Assumptions \ref{ass:1} and \ref{ass:5}. 
Therefore, taking the recursive expansion and $\beta \in (0, 1)$, we have 
\begin{align}
\mathbb{E}\|A_{t}\|^{2} \leq 2(\sigma^{2} + G^{2} + \rho^{2})
\end{align}
\end{proof}

\begin{lemma} \label{lem:A3}
For $i \in [N] $, we have  
\begin{align}
\mathbb{E}\|\nabla_x f_i (x_{t,i}; \mathcal{B}_{t,i}) - g_{t,i} \|^2 \leq& \frac{\sigma^2}{b} \\
\mathbb{E}\|g_{t} - \mathbf{1} \otimes \bar{g}_{t} \|^2
\leq& 12 L^{2} \mathbb{E}\|x_{t-1}-\bar{x}_{t-1}\|^{2} + 6K \zeta^{2}
\end{align}

\end{lemma}
\begin{proof}
(1) we have
\begin{align}
&\mathbb{E}\|\nabla_x f_i (x_{t,i}; \mathcal{B}_{t,i}) - g_{t,i} \|^2 \nonumber\\
=& \mathbb{E}\|\frac{1}{b} \sum_{\xi_{t,i} \in \mathcal{B}_{t,i}}(\nabla f_i (x_{t,i} ; \xi_{t,i}) - g_{t,i}) \|^2 \nonumber\\
=& \frac{1}{b^2} \mathbb{E}\|\sum_{\xi_{t,i} \in \mathcal{B}_{t,i}}(\nabla f_i (x_{t,i} ; \xi_{t,i}) - g_{t,i}) \|^2 \nonumber\\
=& \frac{1}{b^2} \sum_{\xi_{t,i} \in \mathcal{B}_{t,i}} \mathbb{E}\|\nabla f_i (x_{t,i} ; \xi_{t,i}) - g_{t,i}\|^2 \nonumber\\
\leq& \frac{\sigma^2}{b}
\end{align}
where the third equality is due to $ \mathbb{E}_{\mathcal{B}_{t,i}}[\nabla f_i (x_{t,i} ; \xi_{t,i}) - g_{t,i}] = 0$ and the last inequality follows Assumptions \ref{ass:1}. 

(2)
\begin{align}
&\mathbb{E}\|g_{t} - \mathbf{1} \otimes \bar{g}_{t} \|^2 \nonumber\\
=& \sum_{i=1}^{N} \mathbb{E}\|g_{t,i} - \bar{g}_{t} \|^2 \nonumber\\
\leq& 3 \sum_{i=1}^{N}\mathbb{E}[\|g_{t,i}-\nabla f_i(\bar{x}_{t})\|^{2} +  \|\nabla f(\bar{x}_{t}) - \bar{g}_{t} \|^2 + \| \nabla f_i(\bar{x}_{t}) - \nabla f(\bar{x}_{t}) \|^2]\nonumber\\
\leq& 3 \sum_{i=1}^{N}\mathbb{E}[\|g_{t,i} - \nabla f_i(\bar{x}_{t})\|^{2} + \frac{1}{N} \sum_{j=1}^{N}\|\nabla f_j(\bar{x}_{t}) - g_{t,j} \|^2 + \frac{1}{N} \sum_{j=1}^{N}\| \nabla f_i(\bar{x}_{t}) - \nabla f_j(\bar{x}_{t}) \|^2]\nonumber\\
\leq& 6 L^{2} \mathbb{E}\|x_{t}- \mathbf{1} \otimes \bar{x}_{t}\|^{2} + 3 \sum_{i=1}^{N} \frac{1}{N} \sum_{j=1}^{N} \mathbb{E}\|\nabla f_i\left(\bar{x}_{t}\right)-\nabla f_j(\bar{x}_{t})\|^{2} \nonumber\\
\leq& 6 L^{2} \mathbb{E}\|x_{t} - \mathbf{1} \otimes \bar{x}_{t}\|^{2} + 3N \zeta^{2}
\end{align}
where the third inequality is due to Assumption  \ref{ass:2}  and the last inequality is due to Assumption \ref{ass:1}. 
\end{proof}

\begin{lemma} \label{lem:A4}
For $t \in [\lfloor t / q\rfloor q, \lfloor t / q\rfloor (q+1)]$, and $x_t$ is generated from Algorithm \ref{alg:2}, we have  
\end{lemma}
\begin{proof}
(1) if $t = s_t = \lfloor t / q\rfloor q$, we have 
\begin{align}
\sum_{i=1}^{N}\left\|x_{s_{t},i} - \bar{x}_{s_{t}}\right\|^{2}=0 
\end{align}

(2) if  $t \geq \lfloor t / q\rfloor q$, we have
\begin{align}
x_{t,i} = x_{s_{t},i} - \sum_{s=s_t}^{t-1} \eta_{s} A_{s}^{-1} m_{s, i} \quad \bar{x}_{t} = \bar{x}_{s_{t}} - \sum_{s=s_t}^{t-1} \eta_{s} A_{s}^{-1} \bar{m}_{s}   \nonumber
\end{align}

\begin{align}
\sum_{i=1}^{N}\left\|x_{t,i} - \bar{x}_{t}\right\|^{2} &= \sum_{i=1}^{N}\left\|x_{s_{t},i} - \bar{x}_{s_{t}} - \left(\sum_{s = s_t}^{t-1} \eta_{s} A_{s}^{-1} m_{s,i} - \sum_{s=s_{t}}^{t-1}\eta_{s} A_{s}^{-1} \bar{m}_{s}\right)\right\|^{2} \nonumber\\
&= \sum_{i=1}^{N}\left\|\sum_{s=s_{t}}^{t-1} \eta_{s} A_{s}^{-1}\left[m_{s,i}-\bar{m}_{s}\right]\right\|^{2} \nonumber\\
&\leq (q-1) \sum_{s=s_{t}}^{t-1} \eta_{s}^{2} \sum_{i=1}^{N}\|A_{s}^{-1}(m_{s,i}-\bar{m}_{s})\|^{2}
\end{align}
\end{proof}

\begin{lemma} \label{lem:A5}
Suppose the sequence $\{x_t\}_0^{T}$ be generated from Algorithms 1. We have
\begin{align}
\mathbb{E}f(\bar{x}_{t+1}) &  \leq \mathbb{E}f(\bar{x}_{t}) - (\frac{3 \rho}{4 \eta_t} - \frac{L}{2})\mathbb{E}\|\bar{x}_{t+1} - \bar{x}_{t}\|^{2} - \frac{\eta_t}{4\rho}\mathbb{E}\|\nabla f(\bar{x}_{t}) - \bar{m}_{t}\|^{2} + \frac{5 \eta_t}{2 \rho}\mathbb{E}\|\bar{g}_t - \bar{m}_{t} \|^{2} \nonumber\\
&+  \frac{5 \eta_{t}L^{2}}{2 \rho N} \mathbb{E}\|x_t - \mathbf{1} \otimes  \bar{x}_{t}\|^{2}
\end{align}
\end{lemma}
\begin{proof}
\begin{align}
f(\bar{x}_{t+1}) &\leq f(\bar{x}_{t})+\langle\nabla f(\bar{x}_{t}), \bar{x}_{t+1}-\bar{x}_{t}\rangle+\frac{L}{2}\left\|\bar{x}_{t+1}-\bar{x}_{t}\right\|^{2}  \nonumber\\
&=f(\bar{x}_{t})+\underbrace{ \langle\nabla f(\bar{x}_{t})-\bar{m}_{t}, \bar{x}_{t+1}-\bar{x}_{t}\rangle}_{(1)}+\underbrace{\langle \bar{m}_{t}, \bar{x}_{t+1}-\bar{x}_{t}\rangle}_{(2)}+\frac{L}{2}\|\bar{x}_{t+1}-\bar{x}_{t}\|^{2}
\end{align}

For the term (1), by the Cauchy-Schwarz inequality and Young's inequality, we have
\begin{align}
(1) &= \langle\nabla f(\bar{x}_{t})-\bar{m}_{t}, \bar{x}_{t+1}-\bar{x}_{t}\rangle \nonumber\\
&\leq \|\nabla f\left(\bar{x}_{t}\right)-\bar{m}_{t}\|\|\bar{x}_{t+1}-\bar{x}_{t}\|  \nonumber \\
&\leq \frac{\eta_{t}}{\rho}\|\nabla f(\bar{x}_{t})-\bar{m}_{t}\|^{2}+\frac{\rho}{4 \eta_{t}}\|\bar{x}_{t+1}-\bar{x}_{t}\|^{2}
\end{align}

For the term (2), 
$A_{t} = \operatorname{diag}(\sqrt{\bar{v}_{s_t}}+\rho)$,  and $\bar{x}_{t+1} = \bar{x}_{t} - \eta_t A_{t}^{-1} \bar{m}_t$.
According to the definition of $A_t$, and assumption \ref{ass:4}, i.e., $A_t \succ \rho I_{d}$ for any $t\geq 1$,
we have
\begin{align}
\langle \bar{m}_{t}, \frac{1}{\eta_{t}}(\bar{x}_{t}-\bar{x}_{t+1})\rangle \geq \rho\|\frac{1}{\eta_{t}}(\bar{x}_{t}-\bar{x}_{t+1})\|^{2}
\end{align}
Then we obtain
\begin{align}
(2)  = \langle\bar{m}_{t}, \bar{x}_{t+1}-\bar{x}_{t}\rangle \leq - \frac{\rho}{\eta_{t}}\|\bar{x}_{t+1}-\bar{x}_{t}\|^{2}
\end{align}

Then we have 
\begin{align}
f(\bar{x}_{t+1})  &\leq f(\bar{x}_{t})  + \frac{\eta_{t}}{\rho}\|\nabla f(\bar{x}_{t}) - \bar{m}_{t}\|^{2} + \frac{\rho}{4 \eta_{t}}\|\bar{x}_{t+1} - \bar{x}_{t}\|^{2} - \frac{\rho}{\eta_{t}}\|\bar{x}_{t+1}-\bar{x}_{t}\|^{2}+\frac{L}{2}\|\bar{x}_{t+1}-\bar{x}_{t}\|^{2} \nonumber\\
&\leq f(\bar{x}_{t}) -  \frac{\eta_{t}}{4\rho}\|\nabla f(\bar{x}_{t})-\bar{m}_{t}\|^{2} + \frac{5\eta_{t}}{4\rho}\|\nabla f(\bar{x}_{t})-\bar{m}_{t}\|^{2}-( \frac{3\rho}{4\eta_{t}}-\frac{L}{2})\|\bar{x}_{t+1}-\bar{x}_{t}\|^{2}  \nonumber\\
&\leq f(\bar{x}_{t})-( \frac{3\rho}{4\eta_{t}}-\frac{L}{2})\|\bar{x}_{t+1}-\bar{x}_{t}\|^{2} -  \frac{\eta_t}{4\rho}\|\nabla f(\bar{x}_{t})-\bar{m}_{t}\|^{2} + \frac{5 \eta_t}{2 \rho}\|\bar{g}_t - \bar{m}_{t} \|^{2} \nonumber\\
&+  \frac{5 \eta_{t}}{2 \rho} \| \nabla f(\bar{x}_{t})- \bar{g}_t\|^{2}
\end{align}

Taking expectation on both sides and considering the last term
\begin{align}
 \mathbb{E}\| \nabla f(\bar{x}_{t}) - \bar{g}_t \|^{2} &\leq \frac{1}{N} \sum_{i=1}^{N} \mathbb{E}\|\nabla f_i(\bar{x}_{t}) - g_{t,i}\|^{2} \nonumber\\
 &\leq \frac{L^{2}}{N} \sum_{i=1}^{N} \mathbb{E}\|x_{t,i} - \bar{x}_{t}\|^{2} \nonumber\\
 & = \frac{L^{2}}{N}\mathbb{E}\|x_t - \mathbf{1} \otimes \bar{x}_{t}\|^{2}
\end{align}

Therefore, we obtain
\begin{align}
\mathbb{E}f(\bar{x}_{t+1}) &  \leq \mathbb{E}f(\bar{x}_{t}) - (\frac{3 \rho}{4 \eta_t} - \frac{L}{2})\mathbb{E}\|\bar{x}_{t+1} - \bar{x}_{t}\|^{2} - \frac{\eta_t}{4\rho}\mathbb{E}\|\nabla f(\bar{x}_{t}) - \bar{m}_{t}\|^{2} + \frac{5 \eta_t}{2 \rho}\mathbb{E}\|\bar{g}_t - \bar{m}_{t} \|^{2} \nonumber\\
&+  \frac{5 \eta_{t}L^{2}}{2 \rho N} \mathbb{E}\|x_t - \mathbf{1} \otimes \bar{x}_{t}\|^{2}
\end{align}
\end{proof}
\begin{lemma} \label{lem:A6}
Assume that the stochastic partial derivatives $m_{t}$  be generated from Algorithm \ref{alg:2}, we have
\begin{align}
\mathbb{E}\|\bar{m}_{t} - \bar{g}_t\|^2 = (1-\alpha_{t})^{2}\mathbb{E}\|\bar{m}_{t-1} - \bar{g}_{t-1}\|^{2} +  \frac{2(1 - \alpha_t)^2 L^2}{N^2 b} \mathbb{E}\|x_{t} - x_{t-1}\|^2
 +  \frac{2\alpha_{t}^2\sigma^2}{Nb}    
\end{align}
\end{lemma}
\begin{proof}
Recall that $ \bar{m}_{t} = \frac{1}{N}\sum_{i=1}^{N} [\nabla_x f_i(x_{t,i}; \mathcal{B}_{t,i}) + (1 - \alpha_{t})(\bar{m}_{t-1} - \nabla_x f_i (x_{t-1,i}; \mathcal{B}_{t,i}))]$, we have
\begin{align}
\mathbb{E}\|\bar{m}_{t} - \bar{g}_t\|^2 &= \mathbb{E}\|\frac{1}{N}\sum_{i=1}^{N} [\nabla_x f_i (x_{t,i}; \mathcal{B}_{t,i}) + (1-\alpha_{t})(\bar{m}_{t-1}  - \frac{1}{N}\sum_{i=1}^{N} \nabla_x f_i(x_{t-1,i}; \mathcal{B}_{t,i}))] - \bar{g}_t \|^{2} \nonumber\\
&= \mathbb{E}\|\frac{1}{N}\sum_{i=1}^{N} [(\nabla_x f_i (x_{t,i}; \mathcal{B}_{t,i}) - g_{t,i}) -(1-\alpha_{t})(f_i (x_{t-1,i}; \mathcal{B}_{t,i}) - g_{t-1,i} ] +  (1-\alpha_{t})(\bar{m}_{t-1} - \bar{g}_{t-1})\|^2
\end{align}
Given that $\mathbb{E}[(\nabla_x f_i (x_{t,i}; \mathcal{B}_{t,i}) - g_{t,i}) -(1-\alpha_{t})(f_i (x_{t-1,i}; \mathcal{B}_{t,i}) - g_{t-1,i} )] = 0$
\begin{align}
\mathbb{E}\|\bar{m}_{t} - \bar{g}_t\|^2 &= (1-\alpha_{t})^{2}\mathbb{E}\|\bar{m}_{t-1} - \bar{g}_{t-1}\|^{2} + \frac{1}{N^2} \sum_{i=1}^{N}\mathbb{E} \|(\nabla_x f_i (x_{t,i}; \mathcal{B}_{t,i}) - g_{t,i}) \nonumber\\
&-(1-\alpha_{t})(f_i (x_{t-1,i}; \mathcal{B}_{t,i}) - g_{t-1,i})\|^2 \nonumber\\
&=(1-\alpha_{t})^{2}\mathbb{E}\|\bar{m}_{t-1} - \bar{g}_{t-1}\|^{2} + \frac{1}{N^2} \sum_{i=1}^{N} \mathbb{E}\|(1 - \alpha_t)[(\nabla_x f_i (x_{t,i}; \mathcal{B}_{t,i}) - g_{t,i}) \nonumber\\
&- (\nabla_x f_i (x_{t-1,i}; \mathcal{B}_{t,i}) - g_{t-1,i})] + \alpha_t(\nabla_x f_i (x_{t,i}; \mathcal{B}_{t,i}) - g_{t,i}) \|^2 \nonumber\\
&\leq (1-\alpha_{t})^{2}\mathbb{E}\|\bar{m}_{t-1} - \bar{g}_{t-1}\|^{2} + \frac{2(1 - \alpha_t)^2}{N^2} \sum_{i=1}^{N} \mathbb{E}\|(\nabla_x f_i (x_{t,i}; \mathcal{B}_{t,i}) - g_{t,i}) \nonumber\\
&- (\nabla_x f_i (x_{t-1,i}; \mathcal{B}_{t,i}) - g_{t-1,i})\|^2 +  \frac{2\alpha_t^2}{N^2}\sum_{i=1}^{N} \mathbb{E}\|\nabla_x f_i (x_{t,i}; \mathcal{B}_{t,i}) - g_{t,i} \|^2 \nonumber\\
&\leq (1-\alpha_{t})^{2}\mathbb{E}\|\bar{m}_{t-1} - \bar{g}_{t-1}\|^{2} +  \frac{2(1 - \alpha_t)^2}{N^2b^2} \sum_{i=1}^{N}\sum_{\xi_{t,i} \in \mathcal{B}_{t,i}} \mathbb{E}\|\nabla f_i (x_{t,i} ; \xi_{t,i}) \nonumber\\
&-\nabla f_i(x_{t-1, i} ; \xi_{t,i})\|^2 + \frac{2\alpha_t^2 }{N^2}\sum_{i=1}^{N} \mathbb{E}\|\nabla_x f_i (x_{t,i}; \mathcal{B}_{t,i}) - g_{t,i} \|^2 \nonumber\\
&\leq (1-\alpha_{t})^{2}\mathbb{E}\|\bar{m}_{t-1} - \bar{g}_{t-1}\|^{2} +  \frac{2(1 - \alpha_t)^2 L^2}{N^2b} \sum_{i=1}^{N} \mathbb{E}\|x_{t,i} - x_{t-1,i}\|^2
 + \frac{2\alpha_{t}^2\sigma^2}{Nb} \nonumber\\
 &= (1-\alpha_{t})^{2}\mathbb{E}\|\bar{m}_{t-1} - \bar{g}_{t-1}\|^{2} +  \frac{2(1 - \alpha_t)^2 L^2}{N^2 b} \mathbb{E}\|x_{t} - x_{t-1}\|^2
 +  \frac{2\alpha_{t}^2\sigma^2}{Nb} 
 \end{align}
 where the last inequality is due to the Assumption \ref{ass:2} and Lemma \ref{lem:A2}
\end{proof}
\begin{lemma} \label{lem:A7}
Assume that the stochastic partial derivatives $m_t$ be generated from Algorithm \ref{alg:2}, we have
\begin{align}
\frac{15\rho}{72N}\sum_{t=s_t}^{\bar{s}} \eta_{t} \sum_{i=1}^{N} \mathbb{E} \|A_{t}^{-1}(m_{t,i}-\bar{m}_{t}) \|^{2} \leq \frac{\rho}{8} \sum_{t=s_t}^{\bar{s}} \frac{1}{\eta_{t}} \mathbb{E}\|\bar{x}_{t+1} - \bar{x}_t\|^{2} + \left[\frac{\rho \sigma^{2} c^{2}}{8 b L^{2}} + \frac{3 \rho \zeta^{2} c^{2}}{8 L^{2}}\right] \sum_{t=s_t}^{\bar{s}} \eta_{t}^{3}
\end{align}
\end{lemma}
\begin{proof}
\begin{align} \label{eq:30}
\sum_{i=1}^{N}\mathbb{E}\|A_{t}^{-1}(m_{t,i}-\bar{m}_{t}) \|^{2} &\leq \sum_{i=1}^{N}\mathbb{E} \|A_{t}^{-1} \left[ \nabla_x f_i (x_{t,i}; \mathcal{B}_{t,i}) ] + (1-\alpha_{t})(m_{t-1,i} - \nabla_x f_i(x_{t-1,i}; \mathcal{B}_{t,i})) \right. \nonumber\\
&- \frac{1}{N} \sum_{i=1}^{N}[\nabla_x f_i (x_{t,i}; \mathcal{B}_{t,i}) + (1-\alpha_{t})(m_{t-1,i} - \nabla_x f_i(x_{t-1,i}; \mathcal{B}_{t,i}))]] \|^{2} \nonumber\\
&= \sum_{i=1}^{N} \mathbb{E} \|A_{t}^{-1}[(1-\alpha_{t})(m_{t-1,i} - \bar{m}_{t-1}) + [ \nabla_x f_i (x_{t,i}; \mathcal{B}_{t,i}) - \frac{1}{N} \sum_{i=1}^{N}  \nabla_x f_i (x_{t,i}; \mathcal{B}_{t,i}) \nonumber\\
&- (1-\alpha_{t})(\nabla_x f_i (x_{t-1,i}; \mathcal{B}_{t,i}) - \frac{1}{N} \sum_{i=1}^{N}  \nabla_x f_i (x_{t-1,i}; \mathcal{B}_{t,i}))]] \|^{2} \nonumber\\
&\leq (1 + \gamma)(1 - \alpha_{t})^{2} \sum_{i=1}^{N} \mathbb{E}\|A_{t}^{-1}(m_{t-1,i} - \bar{m}_{t-1}) \|^{2} + (1 + \frac{1}{\gamma}) \frac{1}{\rho^{2} } \mathbb{E} \| [\nabla_x f_i (x_{t,i}; \mathcal{B}_{t,i}) \nonumber\\
-& \frac{1}{N} \sum_{i=1}^{N} \nabla_x f_i (x_{t,i}; \mathcal{B}_{t,i})] - (1-\alpha_{t})[ \nabla_x f_i (x_{t-1,i}; \mathcal{B}_{t,i}) - \frac{1}{N} \sum \nabla_x f_i (x_{t-1,i}; \mathcal{B}_{t,i})] \|^{2}.
\end{align}
where the second inequality is due to Young's inequality and $A_t \succ \rho I_{d}$. For the second term, we have 
\begin{align} \label{eq:31}
& \sum_{i=1}^{N} \mathbb{E} \| \nabla_x f_i (x_{t,i}; \mathcal{B}_{t,i}) - \frac{1}{N} \sum_{i=1}^{N} \nabla_x f_i (x_{t,i}; \mathcal{B}_{t,i})
- (1-\alpha_{t})[ \nabla_x f_i (x_{t-1,i}; \mathcal{B}_{t,i}) - \frac{1}{N} \sum_{i=1}^{N} \nabla_x f_i (x_{t-1,i}; \mathcal{B}_{t,i})] \|^{2}\nonumber\\
=&\sum_{i=1}^{N} \mathbb{E} \|[\nabla_x f_i (x_{t,i}; \mathcal{B}_{t,i}) - \frac{1}{N} \sum_{i=1}^{N} \nabla_x f_i (x_{t,i}; \mathcal{B}_{t,i})] - [\nabla_x f_i (x_{t-1,i}; \mathcal{B}_{t,i}) - \frac{1}{N} \sum_{i=1}^{N} \nabla_x f_i (x_{t-1,i}; \mathcal{B}_{t,i})] \nonumber\\
+& a_{t}[\nabla_x f_i (x_{t-1,i}; \mathcal{B}_{t,i}) - \frac{1}{N} \sum_{i=1}^{N} \nabla_x f_i (x_{t-1,i}; \mathcal{B}_{t,i})] \|^{2} \nonumber\\
\leq& 2 \sum_{i=1}^{N} \mathbb{E} \|[\nabla_x f_i (x_{t,i}; \mathcal{B}_{t,i}) - \frac{1}{N} \sum_{i=1}^{N} \nabla_x f_i (x_{t,i}; \mathcal{B}_{t,i})] - [\nabla_x f_i (x_{t-1,i}; \mathcal{B}_{t,i}) \nonumber\\
-& \frac{1}{N} \sum_{i=1}^{N} \nabla_x f_i (x_{t-1,i}; \mathcal{B}_{t,i})]\|^{2}+2 \alpha_{t}^{2} \sum_{i=1}^{N} \mathbb{E}\|\nabla_x f_i (x_{t-1,i}; \mathcal{B}_{t,i}) - \frac{1}{N} \sum_{i=1}^{N} \nabla_x f_i (x_{t-1,i}; \mathcal{B}_{t,i})\|^{2} \nonumber\\
\leq& 2 \sum_{i=1}^{N} \mathbb{E} \| \nabla_x f_i (x_{t,i}; \mathcal{B}_{t,i}) - \nabla_x f_i (x_{t-1,i}; \mathcal{B}_{t,i}) \|^{2} + 2 \alpha_{t}^{2} \sum_{i=1}^{N} \mathbb{E}\| \nabla_x f_i (x_{t-1,i}; \mathcal{B}_{t,i}) - \frac{1}{N} \sum_{i=1}^{N} \nabla_x f_i (x_{t-1,i}; \mathcal{B}_{t,i})\|^{2} \nonumber\\
\leq& 2 L^{2} \mathbb{E} \|x_{t}-x_{t-1}\|^{2} + 2 \alpha_{t}^{2} \sum_{i=1}^{N} \mathbb{E} \|\nabla_x f_i (x_{t-1,i}; \mathcal{B}_{t,i}) - \frac{1}{N} \sum_{i=1}^{N} \nabla_x f_i (x_{t-1,i}; \mathcal{B}_{t,i})\|^{2}
\end{align}
where the second inequality is due to Lemma \ref{lem:A1}. The last inequality is due to Assumption \ref{ass:2}. For the last term, we have
\begin{align} \label{eq:32}
&\sum_{i=1}^{N} \mathbb{E} \|\nabla_x f_i (x_{t-1,i}; \mathcal{B}_{t,i}) - \frac{1}{N} \sum_{j=1}^{N} \nabla_x f_j (x_{t-1,j}; \mathcal{B}_{t,j})\|^{2} \nonumber\\
=& \sum_{i=1}^{N}\mathbb{E}\| [\nabla_x f_i (x_{t-1,i}; \mathcal{B}_{t,i}) - g_{t-1,i}] - \frac{1}{N} \sum_{j=1}^{N} [\nabla_x f_j (x_{t-1,j}; \mathcal{B}_{t,j}) - g_{t-1,j}] + [g_{t-1,i} - \bar{g}_{t-1}] \|^{2} \nonumber\\
\leq & 2 \sum_{i=1}^{N}\mathbb{E} \| [\nabla_x f_i (x_{t-1,i}; \mathcal{B}_{t,i}) - g_{t-1,i}] - \frac{1}{N} \sum_{j=1}^{N} [\nabla_x f_j (x_{t-1,j}; \mathcal{B}_{t,j}) -  g_{t-1,j}]\|^{2} 
+ 2\sum_{i=1}^{N}\mathbb{E}\|g_{t-1,i} - \bar{g}_{t-1} \|^2 \nonumber\\
\leq& 2 \sum_{i=1}^{N}\mathbb{E} \|\nabla_x f_i (x_{t-1,i}; \mathcal{B}_{t,i}) - g_{t-1,i} \|^2 + 2\sum_{i=1}^{N}\mathbb{E}\|g_{t-1,i} - \bar{g}_{t-1} \|^2 \nonumber\\
\leq& \frac{2 N \sigma^{2}}{b} + 6 N \zeta^{2} + 12 L^{2} \mathbb{E}\|x_{t-1} - \mathbf{1} \otimes \bar{x}_{t-1}\|^{2}
\end{align}
where the second inequality is due to Lemma \ref{lem:A1} and the last inequality is due to Lemma \ref{lem:A2}. Therefore, by combining above inequalities \ref{eq:30}, \ref{eq:31}, \ref{eq:32} and the definition of $A_t$, when $\mod(t,q) \neq 0 $ we have 
\begin{align}
\sum_{i=1}^{N} \mathbb{E} \|A_{t}^{-1} (m_{t,i} - \bar{m}_t)\|^{2} & \leq (1 - \alpha_{t})^{2}(1 + \gamma) \sum_{i=1}^{N} \mathbb{E} \|A_{t-1}^{-1} (m_{t-1,i} - \bar{m}_{t-1})\|^{2} + \frac{2 L^{2}}{ \rho^{2}}(1+\frac{1}{\gamma})  \mathbb{E}\|x_{t}-x_{t-1}\|^{2} \nonumber\\
&+ \frac{4 N \sigma^{2}}{b \rho^{2}}(1+\frac{1}{\gamma}) \alpha_{t}^{2}+\frac{12 N}{\rho^{2}} \zeta^{2}(1 + \frac{1}{\gamma}) \alpha_{t}^{2} + 24 L^{2}(1+\frac{1}{\gamma}) \frac{\alpha_{t}^{2}}{\rho^{2}}  \mathbb{E}\|x_{t-1} - \mathbf{1} \otimes \bar{x}_{t-1}\|^{2} \nonumber\\
&\leq (1-\alpha_{t})^{2}(1 + \gamma) \sum_{i=1}^{N} \mathbb{E} \|A_{t-1}^{-1} (m_{t-1,i} - \bar{m}_{t-1})\|^{2} + \frac{4 N \sigma^{2}}{b \rho^{2}}(1+\frac{1}{\gamma}) \alpha_{t}^{2}  \nonumber\\
&+\frac{12 N}{\rho^{2}} \zeta^{2}(1 + \frac{1}{\gamma}) \alpha_{t}^{2} + \frac{2 L^{2}}{ \rho^{2}}(1+\frac{1}{\gamma})\sum_{i=1}^{N} \mathbb{E}\|\eta_{t-1} A_{t-1}^{-1} m_{t-1,i}\|^{2}  \nonumber\\
&+ 24 L^{2}(1+\frac{1}{\gamma}) \frac{\alpha_{t}^{2}}{\rho^{2}} (q-1) \sum_{s=s_{t}}^{t-1} \eta_{s}^{2} \sum_{i=1}^{N}\mathbb{E}\|A_{s}^{-1}(m_{s,i}-\bar{m}_{s})\|^{2} \nonumber\\
&\leq (1-\alpha_{t})^{2}(1 + \gamma) \sum_{i=1}^{N} \mathbb{E} \|A_{t-1}^{-1} (m_{t-1,i} - \bar{m}_{t-1})\|^{2} \nonumber\\
&+ \frac{4 L^{2}}{ \rho^{2}}(1+\frac{1}{\gamma}) \sum_{i=1}^{N} \mathbb{E}\big[\|\eta_{t-1} A_{t-1}^{-1} (m_{t-1,i} - \bar{m}_{t-1})\|^{2} + \|\eta_{t-1} A_{t-1}^{-1} \bar{m}_{t-1}\|^{2}\big] \nonumber\\
&+ \frac{4 N \sigma^{2}}{b \rho^{2}}(1+\frac{1}{\gamma}) \alpha_{t}^{2}+\frac{12 N}{\rho^{2}} \zeta^{2}(1 + \frac{1}{\gamma}) \alpha_{t}^{2} \nonumber\\
&+ 24 L^{2}(1+\frac{1}{\gamma}) \frac{\alpha_{t}^{2}}{\rho^{2}} (q-1) \sum_{s=s_{t}}^{t-1} \eta_{s}^{2} \sum_{i=1}^{N}\mathbb{E}\|A_{s}^{-1}(m_{s,i}-\bar{m}_{s})\|^{2} 
\end{align}
where the second inequality is due to  Lemma \ref{lem:A3}. Then we have 

\begin{align} \label{eq:35}
\sum_{i=1}^{N}\mathbb{E}\|A_{t}^{-1}(m_{t,i}-\bar{m}_{t})\|^{2} &= [(1-\alpha_{t})^{2}(1 + \gamma) + \frac{4 L^{2}}{\rho^{2}}(1+\frac{1}{\gamma})\eta_{t-1}^2]\sum_{i=1}^{N} \mathbb{E} \|A_{t-1}^{-1} (m_{t-1,i} - \bar{m}_{t-1})\|^{2} \nonumber\\
&+ \frac{4 N L^{2}}{\rho^{2}}(1+\frac{1}{\gamma})\eta_{t-1}^2\mathbb{E}\|A_{t-1}^{-1} \bar{m}_{t-1}\|^{2} 
+ \frac{4 N \sigma^{2}}{b \rho^{2}}(1+\frac{1}{\gamma}) \alpha_{t}^{2}+\frac{12 N}{\rho^{2}} \zeta^{2}(1 + \frac{1}{\gamma}) \alpha_{t}^{2} \nonumber\\
&+ 24 L^{2}(1+\frac{1}{\gamma}) \frac{\alpha_{t}^{2}}{\rho^{2}} (q-1) \sum_{s=s_{t}}^{t-1} \eta_{s}^{2} \sum_{i=1}^{N}\mathbb{E}\|A_{s}^{-1}(m_{s,i}-\bar{m}_{s})\|^{2} 
\end{align}
Set $\gamma=\frac{1}{q}$ and $\eta_{t} \leq \frac{\rho}{12 L q}$, and given that $\alpha_t \in (0,1)$, 
\begin{align} \label{eq:36}
(1-\alpha_t)^{2}(1+\gamma) + \frac{4 L^{2}}{\rho^{2}}(1+\frac{1}{\gamma}) \eta_{t-1}^{2} & \leq
1+\frac{1}{q} + \frac{4 L^{2}}{\rho^{2}}(1+q) \eta_{t-1}^{2} \nonumber\\
& \leq 1 + \frac{1}{q} + \frac{q+1}{36q^{2}} \nonumber\\
& \leq 1+\frac{19}{18 q}
\end{align}
Put the \eqref{eq:36}  in \eqref{eq:35}, and considering $\gamma=\frac{1}{q}$ and $\eta_{t} \leq \frac{\rho}{12 L q}$, we have 
\begin{align}
\sum_{i=1}^{N}\mathbb{E}\|A_{t}^{-1}(m_{t,i}-\bar{m}_{t})\|^{2}
&\leq (1+\frac{19}{18 q}) \sum_{i=1}^{N} \mathbb{E}\|A_{t-1}^{-1}(m_{t-1, i} - \bar{m}_{t-1}) \|^{2} + \frac{4 N L^{2}}{\rho^{2}}(1 + \frac{1}{\gamma}) \eta_{t-1}^{2} \mathbb{E}\|A_{t-1}^{-1} \bar{m}_{t-1}\|^{2} \nonumber\\
&+ \frac{4 N \sigma^{2}}{\rho^{2} b}(1 + \frac{1}{\gamma}) \alpha_{t}^{2} + \frac{12 N \zeta^{2}}{\rho^{2}}(1+\frac{1}{\gamma}) \alpha_{t}^{2} \nonumber\\
&+ 24 L^{2}(1+\frac{1}{\gamma}) \frac{\alpha_t^{2}}{\rho^{2}}(q-1) \sum_{s=s_t}^{t-1} \eta_{s}^{2} \sum_{i=1}^{N}\mathbb{E}\| A_{s}^{-1}(m_{s,i}-\bar{m}_{s})\|^{2} \nonumber\\
&\leq (1+\frac{19}{18 q}) \sum_{i=1}^{N}\mathbb{E}\|A_{t-1}^{-1}(m_{t-1, i}-\bar{m}_{t-1})\|^{2} + \frac{2N L}{3 \rho} \eta_{t-1} \mathbb{E}\|A_{t-1}^{-1} \bar{m}_{t-1}  \|^{2} +  \frac{2N\sigma^2 c^2}{3\rho b L} \eta_{t-1}^{3} \nonumber\\
&+ \frac{2N \zeta^{2} c^{2}}{L \rho} \eta_{t-1}^{3} +48 \frac{L^{2} q^{2} c^{2} \eta_{t-1}^{4}}{\rho^{2}}  \sum_{s=s_t}^{t-1} \eta_{s}^{2} \sum_{i=1}^{N} \mathbb{E}\|A_{s}^{-1}(m_{s,i}-\bar{m}_{s}) \|^{2}
\end{align}

We know that when $\mod(t,q) \neq 0$ (i.e.  $t = s_t$), $\sum_{i=1}^{N}\|A_{t}^{-1} (m_{t,i} - \bar{m}_t)\|^{2}=0$

\begin{align}
\sum_{i=1}^{N} \mathbb{E}\|A_{t}^{-1} (m_{t,i}-\bar{m}_{t}) \|^{2} &\leq \frac{2 N L}{3 \rho} \sum_{s=s_t}^{t-1}(1+\frac{19}{18q})^{t-1-s} \eta_{s} \mathbb{E}\|A_{s}^{-1}\bar{m}_{s} \|^{2} + \left[\frac{2N \sigma^{2} c^{2}}{3 \rho b L} + \frac{2N \zeta^{2} c^{2}}{L \rho}\right] \sum_{s=s_t}^{t-1}(1+\frac{19}{18 q})^{t-1-s} \eta_{s}^{3} \nonumber\\
&+ \frac{48 L^{2} q^{2} c^{2}}{\rho^{2}}  \sum_{s=s_ t}^{t-1}(1+\frac{19}{18 q})^{t-1-s} \eta_{s}^{4} \sum_{\bar{s}=s_t}^{s} \eta_{\bar{s}}^{2} \sum_{i=1}^{N} \mathbb{E} \|A_{\bar{s}}^{-1}(m_{\bar{s},i}-\bar{m}_{\bar{s}})\|^{2} \nonumber\\
&\leq \frac{2 N L}{3 \rho} \sum_{s=s_{t}}^{t-1}(1+\frac{19}{18q})^{q} \eta_{s} \mathbb{E} \|A_{s}^{-1}\bar{m}_{s}\|^{2} + \left[\frac{2 N \sigma^{2} c^{2}}{3 \rho b L} + \frac{2N \zeta^{2} c^{2}}{L \rho}\right] \sum_{s=s_{t}}^{t-1}\left(1+\frac{19}{18q}\right)^{q} \eta_{s}^{3} \nonumber\\
&+ \frac{48L^{2} q^{3} c^{2}}{\rho^{2}}(\frac{\rho}{12 L q})^{5}(1+\frac{19}{18 q})^{q} \sum_{s=s_{t}}^{t} \eta_{s} \sum_{i=1}^{N} \mathbb{E}\|A_{s}^{-1}(m_{s,i} - \bar{m}_{s}) \|^{2} \nonumber\\
&\leq \frac{2 N L}{\rho} \sum_{s=s_t}^{t} \eta_{s} \mathbb{E}\|A_{s}^{-1}\bar{m}_{s}\|^{2} + \left[\frac{2 N \sigma^{2} c^{2}}{\rho b L} 
+ \frac{6 N \zeta^{2} c^{2}}{L \rho} \right]\sum_{s=s_t}^{t} \eta_{s}^{3} \nonumber\\
&+ \frac{144 L^{2} q^{3} c^{2}}{\rho^{2}}(\frac{\rho}{12 L q})^5 \sum_{s=s_t}^{t} \eta_{s} \sum_{i=1}^{N} \mathbb{E}\|A_{s}^{-1}(m_{s,i}-\bar{m}_{s}) \|^{2}
\end{align}
where the third inequality is due to $(1+19 / 18q)^{q} \leq e^{19 / 18} \leq 3$. Multiplying $\eta_t$ on both side and summing over $[s_t, \bar{s}]$ in one inner loop, we have
\begin{align}
\sum_{t=s_t}^{\bar{s}} \eta_{t} \sum_{i=1}^{N} \mathbb{E} \|A_{t}^{-1}(m_{t,i}-\bar{m}_{t}) \|^{2}  &\leq \frac{2 N L}{\rho} \sum_{t=s_{t}}^{\bar{s}} \eta_{t} \sum_{s=s_{t}}^{t} \eta_{s} \mathbb{E}\|A_{s}^{-1}\bar{m}_{s} \|^{2} + \left[\frac{2 N \sigma^{2} c^{2}}{\rho b L} + \frac{6 N \zeta^{2} c^{2}}{L \rho} \right] \sum_{t=s_t}^{\bar{s}} \eta_{t} \sum_{s=s_t}^{t} \eta_{s}^{3} \nonumber\\
&+ \frac{144 L^{2} q^{3} c^{2}}{\rho^{2}}(\frac{\rho}{12 L q})^5 \sum_{t=s_t}^{\bar{s}} \eta_{t} \sum_{s=s_t}^{t} \eta_{s} \sum_{i=1}^{N} \mathbb{E}\|A_{s}^{-1}(m_{s,i}-\bar{m}_{s})\|^{2}     
\end{align}
Finally,
\begin{align}
\sum_{t=s_t}^{\bar{s}} \eta_{t} \sum_{i=1}^{N} \mathbb{E} \|A_{t}^{-1}(m_{t,i}-\bar{m}_{t}) \|^{2} 
&\leq \frac{2 N L}{\rho} (\sum_{t=s_t}^{\bar{s}} \eta_{t}) \sum_{t=s_{t}}^{\bar{s}} \eta_{t} \mathbb{E}\|A_{t}^{-1}\bar{m}_{t} \|^{2} + \left[\frac{2 N \sigma^{2} c^{2}}{\rho b L} + \frac{6 N \zeta^{2} c^{2}}{L \rho} \right](\sum_{t=s_t}^{\bar{s}} \eta_{t}) \sum_{t=s_t}^{\bar{s}} \eta_{t}^{3} \nonumber\\
&+ \frac{144 L^{2} q^{3} c^{2}}{\rho^{2}} (\frac{\rho}{12 L q})^{5}(\sum_{t=s_ t}^{\bar{s}} \eta_{t}) \sum_{t=s_t}^{\bar{s}} \eta_{t} \sum_{i=1}^{N} \mathbb{E}\|A_{t}^{-1}(m_{t,i} - \bar{m}_{t})\|^{2} \nonumber\\
&\leq \frac{N}{6} \sum_{t=s_t}^{\bar{s}} \eta_{t} \mathbb{E}\|A_{t}^{-1}\bar{m}_{t}\|^{2} + \left[\frac{N \sigma^{2} c^{2}}{6 b L^{2}}  + \frac{N \zeta^{2} c^{2}}{2 L^{2}} \right] \sum_{t=s_t}^{\bar{s}} \eta_{t}^{3} \nonumber\\
&+ \frac{144 L^{2} q^{4} c^{2}}{\rho^{2}}(\frac{\rho}{12 L q})^{6} \sum_{t=s_{t}}^{\bar{s}} \eta_{t} \sum_{i=1}^{N} \mathbb{E}\|A_{t}^{-1}(m_{t,i}-\bar{m}_{t}) \|^{2}
\end{align}

Therefore,
\begin{align}
[1 - \frac{144 L^{2} q^{4} c^{2}}{\rho^{2}}(\frac{\rho}{12 L q})^{6} ]\sum_{t=s_t}^{\bar{s}} \eta_{t} \sum_{i=1}^{N} \mathbb{E} \|A_{t}^{-1}(m_{t,i}-\bar{m}_{t}) \|^{2} \leq& \frac{N}{6} \sum_{t=s_t}^{\bar{s}} \eta_{t} \mathbb{E}\|A_{t}^{-1}\bar{m}_{t}\|^{2} \nonumber\\
+& \left[\frac{N \sigma^{2} c^{2}}{6 b L^{2}}  + \frac{N \zeta^{2} c^{2}}{2 L^{2}} \right] \sum_{t=s_t}^{\bar{s}} \eta_{t}^{3}
\end{align}
Given that $c \leq \frac{120L^2}{bN\rho^2}$, and $1 - \frac{144 L^{2} q^{4} c^{2}}{\rho^{2}}(\frac{\rho}{12 L q})^{6} \geq \frac{22}{72}$. By multiply $\frac{3\rho}{4N}$ on both size, we have
\begin{align} \label{eq:42}
\frac{15\rho}{72N}\sum_{t=s_t}^{\bar{s}} \eta_{t} \sum_{i=1}^{N} \mathbb{E} \|A_{t}^{-1}(m_{t,i}-\bar{m}_{t}) \|^{2} &\leq \frac{\rho}{8} \sum_{t=s_t}^{\bar{s}} \eta_{t} \mathbb{E}\|A_{t}^{-1}\bar{m}_{t}\|^{2}  \nonumber\\
&+ \left[\frac{\rho \sigma^{2} c^{2}}{8 b L^{2}} + \frac{3 \rho \zeta^{2} c^{2}}{8 L^{2}} \right] \sum_{t=s_t}^{\bar{s}} \eta_{t}^{3}
\end{align}
\end{proof}
\section{Proof of Theorem}
\begin{proof}
Set 
$\eta_{t}=\frac{\rho \bar{h}}{\left(w_t + t\right)^{1 / 3}}, \quad \alpha_{t+1} = c \cdot \eta_t^{2}$,  $ c=\frac{1}{12 L q \bar{h}^{3} \rho^{2}} + \frac{60 L^{2}}{b N \rho^{2}} \leq \frac{120 L^{2}}{b N \rho^{2}}$, $\bar{h}=\frac{ N^{2 / 3}}{L}$ and \\
$w_t$ = max $(\frac{3}{2}, 1728 L^3 q^3 \bar{h}^3 - t)$. So, it is clear that $\eta_t \leq \frac{\rho}{12Lq}$ and 

\begin{align}
\eta_{t}^{-1} - \eta_{t-1}^{-1}  &= \frac{(w_t + t)^{1 / 3}}{\rho \bar{h}} - \frac{(w_{t-1} + t-1)^{1 / 3}}{\rho \bar{h}}  \nonumber\\
&\leq \frac{1}{ 3 \rho \bar{h}(w_t +(t-1))^{2 / 3}} \nonumber\\
&\leq \frac{1}{3\rho \bar{h} (w_t / 3 + t)^{2 / 3}} = \frac{3^{2 / 3}}{3 \rho \bar{h}(w_t + t)^{2 / 3}} \nonumber\\
&= \frac{3^{2 / 3}}{3 \bar{h}^{3}\rho^{3}} \cdot \frac{\rho^{2} \bar{h}^{2}}{(w_t + t)^{2 / 3}} = \frac{3^{2 / 3}}{3 \rho^{3} \bar{h}^{3}} \eta_{t}^{2} \nonumber\\
&\leq \frac{\eta_t}{12 \rho^2 \bar{h}^3 L q}
\end{align}

where the first inequality holds by the concavity of function $f(x)=x^{1/3}$, \emph{i.e.}, $(x+y)^{1/3}\leq x^{1/3} + \frac{y}{3x^{2/3}}$. The second inequality follows that $w_t \geq \frac{3}{2}$. And the last inequality holds by
$\eta_{t} \leq \frac{\rho}{12 L q}$, 

\begin{align}
&\frac{\mathbb{E}\|\bar{m}_{t+1} - \bar{g}_{t+1}\|^2}{\eta_{t}} - \frac{\mathbb{E}\|\bar{m}_{t} - \bar{g}_t\|^2}{\eta_{t-1}}  \nonumber\\
\leq& \left[\frac{(1 - \alpha_{t+1})^{2}}{\eta_{t}} - \frac{1}{\eta_{t-1}}\right] \mathbb{E}\|\bar{m}_{t} - \bar{g}_t\|^{2} + \frac{2 (1-\alpha_{t+1})^{2} L^{2}}{b N^{2} \eta_{t}}  \sum_{i=1}^{N}\mathbb{E}\|x_{t+1} - x_{t}\|^{2} 
+ \frac{2 \alpha_{t+1}^{2} \sigma^{2}}{b N \eta_t} \nonumber\\
\leq& \left[\frac{(1 - \alpha_{t+1})^{2}}{\eta_{t}} - \frac{1}{\eta_{t-1}}\right] \mathbb{E}\|\bar{m}_{t} - \bar{g}_t\|^{2} + \frac{4 (1-\alpha_{t+1})^{2} L^{2}}{b N^{2}} \eta_{t} \sum_{i=1}^{N}\mathbb{E}\|A_{t}^{-1}(m_{t,i} - \bar{m}_{t})\|^{2}\nonumber\\
+& \frac{4(1-\alpha_{t+1})^{2}}{b N} L^{2} \eta_{t}\mathbb{E}\|A_{t}^{-1} \bar{m}_t\|^{2} + \frac{2 \alpha_{t+1}^{2} \sigma^{2}}{b N \eta_t} \nonumber\\
\leq& [\eta_{t}^{-1} - \eta_{t-1}^{-1} - c \eta_{t}] \mathbb{E}\|\bar{m}_{t} - \bar{g}_t\|^2 + \frac{4(1-\alpha_{t+1})^{2} L^{2}}{b N^{2}}\eta_{t}\sum_{i=1}^{N}\mathbb{E} \|A_{t}^{-1}(m_{t,i} -  \bar{m}_{t})\|^{2}] \nonumber\\
+& \frac{4(1-\alpha_{t+1})^{2}}{b N} L^{2} \eta_{t}\mathbb{E}\|A_{t}^{-1} \bar{m}_t\|^{2} + \frac{2 \alpha_{t+1}^{2} \sigma^{2}}{b N \eta_t} \nonumber\\
\leq& - \frac{60 L^{2}}{b N \rho^{2}} \eta_{t}\mathbb{E}\|\bar{m}_{t} - \bar{g}_t\|^2 + \frac{4 L^{2}}{b N^{2}} \eta_{t}\sum_{i=1}^{N}\mathbb{E}\|A_{t}^{-1}(m_{t,i}-\bar{m}_{t})\|^{2}+\frac{4 L^{2} \eta_{t}}{b N} \mathbb{E}\|A_{t}^{-1}\bar{m}_{t}\|^{2} + \frac{2 \sigma^{2} c^{2} \eta_{t}^{3}}{b N}.
\end{align}

Therefore, we have 
\begin{align}
&\frac{b N \rho}{24 L^{2}}[\frac{\mathbb{E}\|\bar{m}_{t+1} - \bar{g}_{t+1}\|^2}{\eta_{t}} - \frac{\mathbb{E}\|\bar{m}_{t} - \bar{g}_t\|^{2}}{\eta_{t-1}}] \nonumber\\
\leq& -\frac{5}{2 \rho} \eta_{t} \mathbb{E}\|\bar{m}_{t} - \bar{g}_t\|^2 + \frac{\rho}{6 N} \eta_{t} \sum_{i=1}^{N}\mathbb{E}\|A_{t}^{-1}(m_{t,i}-\bar{m}_{t})\|^{2}+\frac{\eta_{t} \rho}{6} \mathbb{E}\left\|A^{-1} \bar{m}_{t}\right\|^{2} 
+ \frac{\sigma^{2} c^{2} \eta_{t}^{3} \rho}{12 L^{2}}
\end{align}

Next, we define a Lyapunov function, for any $t \leq 1$, we have 

$\Gamma_{t}=f(\bar{x}_{t})+\frac{b N \rho}{24 L^{2}} \frac{\|\bar{m}_{t} - \bar{g}_t\|^{2}}{\eta_{t-1}}$

\begin{align}
\mathbb{E}[\Gamma_{t+1}-\Gamma_{t}] &= \mathbb{E}[f(\bar{x}_{t+1})-f(\bar{x}_{t}) + \frac{\rho b N}{24 L^{2}}(\frac{\|\bar{m}_{t+1} - \bar{g}_{t+1}\|^{2}}{\eta_{t}}-\frac{\|\bar{m}_{t} - \bar{g}_t\|^{2}}{\eta_{t-1}})] \nonumber\\
&\leq - (\frac{3 \rho}{4 \eta_{t}} - \frac{L}{2})\mathbb{E}\|\bar{x}_{t+1} - \bar{x}_{t}\|^{2} -  \frac{\eta_{t}}{4\rho}\mathbb{E}\|\nabla f(\bar{x}_{t}) - \bar{m}_{t}\|^{2} + \frac{5 \eta_{t} L^{2}(q-1)}{2\rho N} \sum_{s=s_t}^{t} \eta_{s}^{2} \sum_{i=1}^{N} \mathbb{E}\|A_{s}^{-1}(m_{s,i} - \bar{m}_{s})\|^{2} \nonumber\\
&+ \frac{\rho}{6 N} \eta_{t} \sum_{i=1}^{N} \mathbb{E}\|A_{t}^{-1}(m_{t,i}-\bar{m}_{t})\|^{2} + \frac{\eta_{t} \rho}{6} \mathbb{E}\|A_{t}^{-1} \bar{m}_{t}\|^{2}+\frac{\sigma^{2} c^{2} \eta_{t}^{3} \rho}{12 L^{2}} \nonumber\\
&= - (\frac{3 \rho}{4 \eta_{t}}-\frac{L}{2})\mathbb{E}\|\bar{x}_{t+1}-\bar{x}_{t}\|^{2} -  \frac{\eta_t}{4\rho}\mathbb{E}\|\nabla f\left(\bar{x}_{t}\right)-\bar{m}_{t}\|^{2} + \frac{5 \eta_{t} L^{2}(q-1)}{2 \rho N} \sum_{s=s_{t}}^{t} \eta_{s}^{2} \sum_{i=1}^{N}\mathbb{E}\|A_{s}^{-1}(m_{s,i}-\bar{m}_{s})\|^{2} \nonumber\\
&+ \frac{\rho}{6 N} \eta_{t} \sum_{i=1}^{N}\mathbb{E}\|A_{t}^{-1}(m_{t,i}-\bar{m}_{t})\|^{2}+\frac{\rho}{6 \eta_{t}} \mathbb{E}\|\bar{x}_{t+1}-\bar{x}_{t}\|^{2}+\frac{\sigma^{2} c^{2} \eta_{t}^{3} \rho}{12 L^{2}}  \nonumber\\
&\leq - \frac{5 \rho}{12 \eta_{t}}\mathbb{E}\|\bar{x}_{t+1} - \bar{x}_{t}\|^{2} -  \frac{\eta_{t}}{4\rho}\mathbb{E}\|\nabla f(\bar{x}_{t})-\bar{m}_{t}\|^{2} + \frac{5\eta_{t} L^{2}(q-1)}{2 \rho N} \sum_{s=s_{t}}^{t} \eta_{s}^{2} \sum_{i=1}^{N} \mathbb{E}\|A_{s}^{-1}(m_{s,i}-\bar{m}_{s})\|^{2} \nonumber\\
&+ \frac{\rho}{6 N} \eta_{t} \sum_{i=1}^{N} \mathbb{E}\|A_{t}^{-1}(m_{t,i}-\bar{m}_{t})\|^{2} + \frac{\sigma^{2} c^{2} \eta_{t}^{3} \rho}{12 L^{2}} 
\end{align}

where the first inequality holds by Lemma \ref{lem:A3} and $\eta_{t} \leq \frac{\rho}{12Lq}$, and the second inequality holds by $ \frac{L}{2} \leq \frac{\rho}{24 \eta_{t} q} \leq \frac{\rho}{24 \eta_{t}}$. Summing the above over $t=s_{t}$ to $\bar{s}, \bar{s} \in [\lfloor t / q \rfloor q, ( \lfloor t / q \rfloor + 1) q]$, we have 
\begin{align} \label{eq:47}
\mathbb{E}[\Gamma_{\bar{s}+1}-\Gamma_{s_t}] &\leq \sum_{t = s_t}^{\bar{s}} [-\frac{5 \rho}{12 \eta_{t}}\mathbb{E}\|\bar{x}_{t+1}-\bar{x}_{t}\|^{2} -  \frac{\eta_{t}}{4\rho}\mathbb{E}\|\nabla f(\bar{x}_{t}) - \bar{m}_{t}\|^{2}] + \sum_{t = s_{t}}^{\bar{s}} \frac{\sigma^{2} c^{2} \eta_{t}^{3}}{12 L^{2}} \rho \nonumber\\
&+ \sum_{t = s_t}^{\bar{s}} \frac{5\eta_{t} L^{2}(q-1)}{2\rho N} \sum_{s=s_{t}}^{t} \eta_{s}^{2} \sum_{i=1}^{N}\mathbb{E}\|A_{s}^{-1}(m_{s,i} - \bar{m}_{s})\|^{2} + \frac{\rho}{6 N} \sum_{t=s_t}^{\bar{s}} \eta_{t} \sum_{i=1}^{N} \mathbb{E}\|A_{t}^{-1}(m_{t,i} - \bar{m}_{t})\|^{2} \nonumber\\
&\leq \sum_{t = s_t}^{\bar{s}} [-\frac{5 \rho}{12 \eta_{t}}\mathbb{E}\|\bar{x}_{t+1}-\bar{x}_{t}\|^{2}-\frac{1}{4} \frac{\eta_{t}}{\rho}\mathbb{E}\|\nabla f(\bar{x}_{t})-\bar{m}_{t}\|^{2}]+\sum_{t=s_t}^{\bar{s}} \frac{\sigma^{2} c^{2} \eta_{t}^{3} \rho}{12 L^{2}} \nonumber\\
&+ \frac{5 L^{2}(q-1)}{2 \rho N}(\sum_{t=s_{t}}^{\bar{s}} \eta_{t}) \sum_{t=s_{t}}^{\bar{s}} \eta_{t}^{2} \sum_{i=1}^{N} \mathbb{E}\|A_{t}^{-1}(m_{t,i}-\bar{m}_{t})\|^{2} +  \frac{\rho}{6 N}\sum_{t=s_t}^{\bar{s}} \eta_{t} \sum_{i=1}^{N} \mathbb{E}\|A_{t}^{-1} (m_{t,i}-\bar{m}_{t})\|^{2} \nonumber\\ 
&\leq \sum_{t = s_t}^{\bar{s}} [-\frac{5 \rho}{12 \eta_{t}}\mathbb{E}\|\bar{x}_{t+1}-\bar{x}_{t}\|^{2}-\frac{1}{4} \frac{\eta_{t}}{\rho}\mathbb{E}\left\|\nabla f(\bar{x}_{t}\right)-\bar{m}_{t}\|^{2}] + \frac{\rho}{6 N} \sum_{t=s_{t}}^{\bar{s}} \eta_{t} \sum_{i=1}^{N}\mathbb{E}\left\|A_{t}^{-1}\left(m_{t,i}-\bar{m}_{t}\right)\right\|^{2}  \nonumber\\
&+ \frac{5 L^{2}(q-1)}{2 \rho N}\left(q \times \frac{\rho}{12 L q} \times \frac{\rho}{12 L q}\right) \sum_{t=s_{t}}^{\bar{s}} \eta_{t} \sum\mathbb{E}\|A_{t}^{-1}\left(m_{t,i} - \bar{m}_{t}\right)\|^{2} +\sum_{t= s_t}^{\bar{s}} \frac{\sigma^{2} c^{2} \eta_{t}^{3} \rho}{12 L^{2}}
\nonumber\\
&\leq \sum_{t=s_{t}}^{\bar{s}} [-\frac{5 \rho}{12 \eta_{t}}\mathbb{E}\|\bar{x}_{t+1}-\bar{x}_{t}\|^{2} -  \frac{\eta_{t}}{4\rho}\mathbb{E}\|\nabla f(\bar{x}_{t})-\bar{m}_{t}\|^{2}]+\sum_{t=s_t}^{\bar{s}} \frac{\sigma^{2} c^{2} \eta_{t}^{3} \rho}{12 L^{2}} \nonumber\\
&+ \frac{9}{48} \frac{\rho}{N} \sum_{t=s_{t}}^{\bar{s}} \eta_{t} \sum_{i=1}^{N} \mathbb{E}\|A_{t}^{-1}(m_{t,i}-\bar{m}_{t})\|^{2} \nonumber\\
&\leq \sum_{t=s_{t}}^{\bar{s}} [-\frac{5 \rho}{12 \eta_{t}}\mathbb{E}\|\bar{x}_{t+1}-\bar{x}_{t}\|^{2} -  \frac{\eta_{t}}{4\rho}\mathbb{E}\|\nabla f(\bar{x}_{t})-\bar{m}_{t}\|^{2}]+\sum_{t=s_t}^{\bar{s}} \frac{\sigma^{2} c^{2} \eta_{t}^{3} \rho}{12 L^{2}} \nonumber\\
&+ \frac{\rho}{8} \sum_{t=s_t}^{\bar{s}} \eta_{t} \mathbb{E}\|A_{t}^{-1}\bar{m}_{t}\|^{2} + \left[\frac{\rho \sigma^{2} c^{2}}{8 b L^{2}} + \frac{3 \rho \zeta^{2} c^{2}}{8 L^{2}} \right] \sum_{t=s_t}^{\bar{s}} \eta_{t}^{3} 
\end{align} 
where the last inequality holds by Lemma \ref{lem:A7} and the fact that $\frac{9}{48} < \frac{15}{72}$. Then summing over from the beginning, we have  
\begin{align}
\mathbb{E}\left[\Gamma_{T}-\Gamma_{0}\right] &\leq \sum_{t=0}^{T-1}\left[-\frac{\rho}{4 \eta_{t}}\mathbb{E}\left\|\bar{x}_{t+1}-\bar{x}_{t}\right\|^{2} -  \frac{\eta_{t}}{4\rho}\mathbb{E}\left\|\nabla\left(\bar{x}_{t}\right)-\bar{m}_{t}\right\|^{2}\right]+\sum_{t=0}^{T-1} \frac{\sigma^{2} c^{2} \eta_{t}^3 \rho}{12 L^{2}} \nonumber\\
&+ \frac{\rho \sigma^{2} c^{2}}{8 b L^{2}} \sum_{t=0}^{T-1} \eta_{t}^{3} + \frac{3 \rho \zeta^{2} c^{2}}{8 L^{2}} \sum_{t=0}^{T-1} \eta_{t}^{3}
\end{align}
Then we move terms and obtain
\begin{align}
\sum_{t=0}^{T-1}\mathbb{E}\left[\frac{\rho}{4 \eta_{t}}\left\|\bar{x}_{t+1}-\bar{x}_{t}\right\|^{2}+ \frac{\eta_{t}}{4\rho}\left\|\nabla f\left(\bar{x}_{t}\right)-\bar{m}_{t}\right\|^{2}\right] 
&\leq 
\mathbb{E}\left[\Gamma_{0}-\Gamma_{T}\right] + \frac{5 \sigma^{2} c^{2} \rho}{24 L^{2}} \sum_{t=0}^{T-1}\eta^{3}_t +\frac{3 \rho \zeta^{2} c^{2}}{8 L^{2}} \sum_{t=0}^{T-1}\eta_{t}^{3} \nonumber\\
&\leq 
\mathbb{E}\left[f(\bar{x}_0)- f^{*}\right]+\frac{b N \rho}{24 L^{2}} \frac{\|\bar{m}_{0} - \bar{g}_0\|^{2}}{\eta_{0}} \nonumber\\
&+\frac{5 \sigma^{2} c^{2} \rho}{24 L^{2}} \sum_{t=0}^{T-1}\eta^{3}_t + \frac{3 \rho \zeta^{2} c^{2}}{8 L^{2}} \sum_{t=1}^{T-1}\eta_{t}^{3}
\end{align}

Then consider that
$\sum_{t=0}^{T-1} \eta_{t}^{3} = \sum_{t=0}^{T-1} \frac{\rho^{3} \bar{h}^{3}}{w_t + t} 
\leq \sum_{t=0}^{T-1} \frac{\rho^{3} \bar{h}^{3}}{1 + t} \leq \rho^{3} \bar{h}^{3} (\ln T + 1)$, since $w_t \geq \frac{3}{2}>1$. Taking Lemma \ref{lem:A3} and dividing the above by $\rho \eta_T T$, we have
\begin{align} \label{eq:50}
&\frac{1}{T} \sum_{t=0}^{T-1}\mathbb{E}\left[ \frac{1}{4\eta_{t}^{2}}\left\|\bar{x}_{t+1,i}-\bar{x}_{t}\right\|^{2}+\frac{1}{4 \rho^{2}}\left\|\nabla f\left(\bar{x}_{t}\right)-\bar{m}_{t}\right\|^{2}\right] \nonumber\\
\leq& \frac{\mathbb{E}\left[f(\bar{x}_0)- f^{*}\right]}{\eta_{T} T \rho}+\frac{ \sigma^{2} b}{\eta_{T} T 24 L^{2} B \eta_{0}} +\frac{\rho^{3}}{\eta_T T L^{2}}\left[\frac{5 \sigma^{2}}{24}+\frac{3 \zeta^{2}}{8}\right] c^{2} \bar{h}^{3} (\ln T +1)
\end{align}
    
For the first term in \ref{eq:50}, we have 
\begin{align}
&\frac{1}{\eta_{T} T}=\frac{\left(w_T + T\right)^{1 / 3}}{\rho \bar{h} T} \leq \frac{w_T^{1/3}}{\rho \bar{h} T}+\frac{1}{\rho \bar{h} T^{2/ 3}} \leq \frac{12 L q}{\rho T}+\frac{L}{\rho(N T)^{2 / 3}}
\end{align}

For the second term in \ref{eq:50}, set B = qb, we have
\begin{align}
\frac{\sigma^{2} b}{\eta_{T} T 24 L^{2} B \eta_{0}} &\leq \left(\frac{12 L q}{\rho T} + \frac{L}{\rho( N T)^{2 / 3}}\right) \times \frac{\sigma^{2}}{24 L^{2}} \times \frac{b w_{0}^{1 / 3}}{B \bar{h} \rho} \nonumber\\
&\leq \left(\frac{12 L q}{\rho T}+\frac{L}{\rho(N T)^{2 / 3}}\right) \times \frac{\sigma^{2}}{24 L^{2}} \times \frac{12 L q b}{B \rho} \nonumber\\
&\leq \frac{6 b q^{2} \sigma^{2}}{B T \rho^{2}}+\frac{b q \sigma^{2}}{2(N T)^{2/3} B \rho^{2}}=\frac{6 q \sigma^{2}}{T \rho^{2}} + \frac{\sigma^{2}}{2(N T)^{2/3 }\rho^{2}}
\end{align}
For the third term
\begin{align}
\frac{\rho^{3} c^{2} \bar{h}^{3}}{8 \eta_{T} T L^{2}} &\leq \left(\frac{12 L q}{\rho T}+\frac{L}{\rho(N T)^{2 / 3}}\right) \times\left(\frac{120 L^{2}}{b N \rho^{2}}\right)^{2} \times \frac{N^{2} \rho^{3}}{8 L^{3} \cdot L^{2}} \nonumber\\
&=\left(\frac{12 L q}{\rho T}+\frac{L}{\rho(N T)^{2 / 3}}\right) \times\left(\frac{1800}{b^2\rho L }\right)\nonumber\\
&=\frac{12^{2} \times 150 q}{b^2\rho^{2} T}+\frac{1800}{b^2 \rho^{2}(N T)^{2 / 3}}
\end{align}

Let $\mathcal{M}_t = \frac{1}{4\eta_{t}^{2}} \left\|\bar{x}_{t+1}-\bar{x}_{t}\right\|^{2}+\frac{1}{4 \rho^{2}}\left\|\nabla f\left(\bar{x}_{t}\right)-\bar{m}_{t}\right\|^{2}$
\begin{align}
\frac{1}{T} \sum_{t=0}^{T-1}\mathbb{E}[\mathcal{M}_t] =&
\frac{1}{T} \sum_{t=0}^{T-1}\mathbb{E}\left[\frac{1}{4\eta_{t}^{2}} \left\|\bar{x}_{t+1}-\bar{x}_{t}\right\|^{2}+\frac{1}{4 \rho^{2}}\left\|\nabla f\left(\bar{x}_{t}\right)-\bar{m}_{t}\right\|^{2}\right] \nonumber\\
&\leq \left[\frac{12 L q}{\rho T}+\frac{L}{\rho(N T)^{2 / 3}}\right] \mathbb{E}\left[f\left(\bar{x}_{0}\right)-f^{*}\right]+\frac{6 q \sigma^{2}}{T \rho^{2}}+\frac{\sigma^{2}}{2(N T)^{2 / 3 } \rho^{2}}\nonumber\\
&+\left[\frac{12^{2} \times 150 q}{b^2 \rho^{2} T}+\frac{1800}{b^2\rho^{2}(N T)^{2 / 3}}\right]\left[\frac{5 \sigma^{2}}{3}+\frac{3 \zeta^{2}}{2}\right] (\ln T +1)
\end{align}


and, if we  let b as $O(1) (b \geq 1)$, and choose $q=\left(T / N^{2}\right)^{1 / 3}$. To let the right hand is less than $\varepsilon^2$, we get $T = O(N^{-1}\varepsilon^{-3})$
and $\frac{T}{q} = (N T)^{2 / 3} = \varepsilon^{-2}$.

Then with Jensen's inequality:
\begin{align}
 &\frac{1}{\eta_{t}}\left\|\bar{x}_{t}-\bar{x}_{t+1}\right\|+\frac{1}{\rho}\left\|\nabla f\left(\bar{x}_{t}\right)-\bar{m}_{t}\right\| \nonumber\\  
=& \left\|A_{t}^{-1} \bar{m}_{t}\right\| + \frac{1}{\rho}\left\|\nabla f(\bar{x}_{t}) - \bar{m}_{t}\right\| \nonumber\\
=& \frac{1}{\|A_t\|}\left\|A_{t}\right\|\left\|A_{t}^{-1} \bar{m}_{t}\right\| + \frac{1}{\rho}\left\|\nabla f\left(\bar{x}_{t}\right)-\bar{m}_{t}\right\| \nonumber\\
\geq& \frac{1}{\|A_t\|}\left\|\bar{m}_{t}\right\| + \frac{1}{\|A_t\|}\left\|\nabla f\left(\bar{x}_{t}\right)-\bar{m}_{t}\right\| \nonumber\\
\geq& \frac{1}{\|A_t\|}\left\|\nabla f\left(\bar{x}_{t}\right)\right\|
\end{align}
where the first inequality holds by and $\|A_t\| \geq \rho$. Then we have 
\begin{align}
 \frac{1}{T} \sum_{t=1}^{T} \mathbb{E}\|\nabla f(\bar{x}_{t})\| &\leq \frac{1}{T} \sum_{t=1}^{T} \mathbb{E} \|A_{t}\|\left[\frac{1}{\eta_{t}}\left\|\bar{x}_{t}-\bar{x}_{t+1}\right\|^{2}+\frac{1}{\rho}\left\|\nabla f\left(\bar{x}_{t}\right)-\bar{m}_{t}\right\|\right] \nonumber\\    
&\leq \frac{1}{T} \sum_{t=1}^{T} \mathbb{E}\left[\frac{\lambda}{2}\left\|A_{t}\right\|^2 + \frac{1}{2 \lambda} \left[\frac{1}{\eta_{t}}\left\|\bar{x}_{t}-\bar{x}_{t+1}\right\|^{2}+\frac{1}{\rho}\left\|\nabla f\left(\bar{x}_{t}\right)-\bar{m}_{t}\right\|\right]^2\right]  \nonumber\\
&= \frac{\lambda}{2}\frac{1}{T} \sum_{t=1}^{T} \mathbb{E}\left\|A_{t}\right\|^2 + \frac{1}{2 \lambda} \frac{1}{T} \sum_{t=1}^{T} \mathbb{E}\left[\frac{1}{\eta_{t}}\left\|\bar{x}_{t}-\bar{x}_{t+1}\right\|^{2}+\frac{1}{\rho}\left\|\nabla f\left(\bar{x}_{t}\right)-\bar{m}_{t}\right\|\right]^2  \nonumber\\
&= \sqrt{\frac{1}{T} \sum_{t=1}^{T} \mathbb{E}\left\|A_{t}\right\|^{2}} \sqrt{\frac{1}{T} \sum_{t=1}^{T} \mathbb{E} \left[\frac{1}{\eta_{t}}\left\|\bar{x}_{t}-\bar{x}_{t+1}\right\|+\frac{1}{\rho}\left\|\nabla f\left(\bar{x}_{t}\right)-\bar{m}_{t}\right\|\right]^2} \nonumber\\
&\leq \sqrt{\frac{1}{T} \sum_{t=1}^{T} \mathbb{E} \left\|A_{t}\right\|^{2}} \sqrt{\frac{1}{T}\sum_{t=1}^{T} \mathbb{E} \left[\frac{2}{\eta_{t}^{2}}\left\|\bar{x}_{t} - \bar{x}_{t+1}\right\|^{2} + \frac{2}{\rho^{2}}\left\|\nabla f\left(\bar{x}_{t}\right)-\bar{m}_{t}\right\|^{2}\right]} \nonumber\\
&\leq 4\sqrt{(\sigma^{2} + G^{2} + \rho^{2})} \sqrt{\frac{1}{T} \sum_{t=0}^{T-1}\mathbb{E}[\mathcal{M}_t}] 
\end{align}
where $\lambda = \sqrt{\frac{1}{T} \sum_{t=1}^{T} \mathbb{E} \left[\frac{1}{\eta_{t}}\left\|\bar{x}_{t}-\bar{x}_{t+1}\right\|+\frac{1}{\rho}\left\|\nabla f\left(\bar{x}_{t}\right)-\bar{m}_{t}\right\|\right]^2} / \sqrt{\frac{1}{T} \sum_{t=1}^{T} \mathbb{E}\left\|A_{t}\right\|^{2}} $
\end{proof}

\section{Model Architectures of Image Classification Task}
\vspace*{-6pt}
\begin{table*} [h]
  \centering
  \vspace*{-6pt}
  \caption{Model Architecture for MNIST \cite{huang2021efficient}}
  \label{tab:3}
    \vspace*{-6pt}
\begin{tabular}{ll}
\hline Layer Type & Shape \\
\hline Convolution + ReLU & $5 \times 5 \times 20 $ \\
Max Pooling & $2 \times 2$ \\
Convolution + ReLU & $5 \times 5 \times 50$ \\
Max Pooling & $2 \times 2$ \\
Fully Connected + ReLU & 500 \\
Fully Connected + ReLU & 10 \\
\hline
\end{tabular}
\end{table*}
\vspace*{-6pt}
\begin{table*} [h]
  \centering
 \vspace*{-6pt}
  \caption{ Model Architecture for Fashion-MNIST \citep{nouiehed2019solving}}
  \label{tab:4}
 \vspace*{-6pt}
\begin{tabular}{ll}
\hline Layer Type & Shape \\
\hline Convolution + Tanh & $3 \times 3 \times 5 $ \\
Max Pooling & $2 \times 2$ \\
Convolution + Tanh & $3 \times 3 \times 10$ \\
Max Pooling & $2 \times 2$ \\
Fully Connected + Tanh & 100 \\
Fully Connected + Tanh & 10 \\
\hline
\end{tabular}
\end{table*}
\vspace*{-6pt}
\begin{table*} [ht]
  \centering
    \vspace*{-6pt}
  \caption{ Model Architecture for CIFAR-10 \citep{huang2021efficient}}
  \label{tab:5}
    \vspace*{-6pt}
\begin{tabular}{llc}
\hline Layer Type & Shape & padding\\
\hline Convolution + ReLU & $3 \times 3 \times 16 $ & 1 \\
Max Pooling & $2 \times 2$ \\
Convolution + ReLU & $3 \times 3 \times 32$ & 1\\
Max Pooling & $2 \times 2$ \\
Convolution + ReLU & $3 \times 3 \times 64$ & 1\\
Max Pooling & $2 \times 2$ \\
Fully Connected + ReLU & 512 \\
Fully Connected + ReLU & 64 \\
Fully Connected + ReLU & 10 \\
\hline
\end{tabular}
\end{table*}
\end{document}